\documentclass{article}

\usepackage[preprint]{neurips_2026}

\usepackage[utf8]{inputenc} % allow utf-8 input
\usepackage[T1]{fontenc}    % use 8-bit T1 fonts
\usepackage{hyperref}       % hyperlinks
\usepackage{url}            % simple URL typesetting
\usepackage{booktabs}       % professional-quality tables
\usepackage{amsfonts}       % blackboard math symbols
\usepackage{nicefrac}       % compact symbols for 1/2, etc.
\usepackage{microtype}      % microtypography
\usepackage{xcolor}         % colors

\usepackage{natbib}
\usepackage{amsmath}
\usepackage{amssymb}
\usepackage{mathtools}
\usepackage{amsthm}

\usepackage{enumitem}

\usepackage{algorithm}
\usepackage{algpseudocode}
\algnewcommand{\INPUT}{\State \textbf{Input: }}
\algnewcommand{\OUTPUT}{\State \textbf{Output: }}
\algnewcommand{\STATE}{\State}
\algnewcommand{\WHILE}[1]{\While{#1}}
\algnewcommand{\ENDWHILE}{\EndWhile}
\allowdisplaybreaks
\usepackage{caption}
\usepackage{wrapfig}

\usepackage{microtype}
\usepackage{graphicx}
\usepackage{subcaption}
\usepackage{booktabs}

\usepackage[capitalize,noabbrev]{cleveref}

\theoremstyle{plain}
\newtheorem{theorem}{Theorem}[section]
\newtheorem{proposition}[theorem]{Proposition}
\newtheorem{lemma}[theorem]{Lemma}
\newtheorem{corollary}[theorem]{Corollary}
\theoremstyle{definition}

\theoremstyle{remark}

\title{Path-independent Flow Matching for Multi-parameter Generative Dynamics}

\author{
Francisco~Téllez\thanks{Equal contribution} \\
Université de Montréal, Mila
\And
AmirHossein~Zamani\footnotemark[1] \\
Concordia University, Mila
\And
Philippe~Martin\footnotemark[1] \\
Université de Montréal, Mila
\And 
Shuang~Ni \\
Université de Montréal, Mila
\And
Guy~Wolf \\
Université de Montréal, Mila
\And 
Eugene~Belilovsky \\
Concordia University, Mila
\And
Sina~Sanjari\thanks{Co-senior authors} \\
Royal Military College of Canada
\And
Yanlei~Zhang\footnotemark[2] \\
Queen's University, Mila\\
Correspondance: \texttt{zhang.yanlei@queensu.ca}
}

\begin{document}
\maketitle

%\renewcommand{\thefootnote}{}
%\footnotetext{%
%\textsuperscript{*} Equal %contribution. %\textsuperscript{$\dag$} Co-senior %authors%
%}

\begin{abstract}
  Flow Matching is a powerful framework for learning transport maps between probability distributions. Yet its standard single-parameter formulation is not designed to capture multi-parameter variations where the resulting transport should be path-independent. Path independence is crucial because it ensures that transformations depend only on the initial and target distributions, not on the specific path. In this work, we introduce \textit{Path-independent Flow Matching (PiFM)}, a method for learning vector fields whose induced flows yield path-independent transport between distributions. We show that PiFM generalizes Flow Matching to higher-dimensional parameter domains while enforcing structural conditions that ensure consistency of composed transformations. In addition, we show that, under suitable assumptions, PiFM approximates the Wasserstein barycenter, linking the framework to a notion of distributional interpolation. To enable practical training, we propose a tractable, simulation-free objective that regresses onto multi-parameter conditional probability paths. We showcase empirically that PiFM outperforms other approaches on both synthetic and real world data in interpolating path-independent trajectories and generating desired out of distribution samples. %Our code is available at \url{https://anonymous.4open.science/r/Pi_flow_matching-C83F/}.
\end{abstract}

\vspace{-0.8em}

\section{Introduction}

\vspace{-0.8em}

Learning transformations between probability distributions is a central problem in modern machine learning, with applications spanning generative modeling, domain adaptation, trajectory inference, and scientific discovery. In these settings, one seeks not only to match target distributions, but also to understand how distributions relate through structured, continuous transformations such as the flows of vector fields. Recent neural approaches achieve this by reducing population-level objectives to sample-wise computations, enabling scalable training via stochastic optimization. This paradigm underlies VAEs \citep{kingma2022autoencoding, Diederik_2019}, GANs \citep{goodfellow2014generative, Karthika2021}, diffusion models \citep{ho2020denoising, song2021score}, and continuous-time transport methods \citep{chen2019neural, grathwohl2018ffjord, lipman2023, tong2023improving, tong2024simulation}, and has led to growing interest in modeling distributions through dynamical systems.

A broad class of methods, including continuous normalizing flows \citep{chen2019neural, grathwohl2018ffjord}, diffusion and score-based models \citep{ho2020denoising, song2021score, koehler2022statistical, berner2022an} , and Schrödinger bridges \citep{leonard2013survey, debortoli2023diffusionschrodinger, liu2023i2sbimage}, can be interpreted as learning time-varying transport dynamics that push a source distribution toward a target. Among these, Flow Matching \citep{lipman2023, liu2022} provides a particularly simple and simulation-free training principle by directly regressing neural vector fields toward reference conditional velocities induced by a prescribed probability path. Conditional Flow Matching \citep{tong2023improving, tong2024simulation} and related extensions further generalize this idea, unifying several continuous-time generative frameworks by allowing transport between a broader family of source and target distributions.

Despite their success, existing flow matching methods are inherently limited to a single time parameter, implicitly assuming that all variations can be ordered along a single temporal axis. In many real-world settings, however, data are influenced by multiple independent processes that evolve concurrently and should compose in a path-independent manner. This arises naturally in settings such as biological perturbations acting on time-evolving cellular populations in single-cell genomics \citep{lotfollahi2023predicting, jiang2025systematic}, compositional attributes like shape and color in images \citep{wang2024scene}, and modular transformations in physical systems \citep{neary2023compositional}. Naively composing independently trained flows generally leads to non-commutative dynamics, where the final distribution depends on the order of application, limiting interpretability and downstream geometric reasoning as tasks such as analogy completion or attribute extrapolation become ill-defined when transformations fail to commute. 

\begin{figure*}[t]
\centering
\begin{subfigure}{0.30\textwidth}
\centering
\includegraphics[scale=0.28]{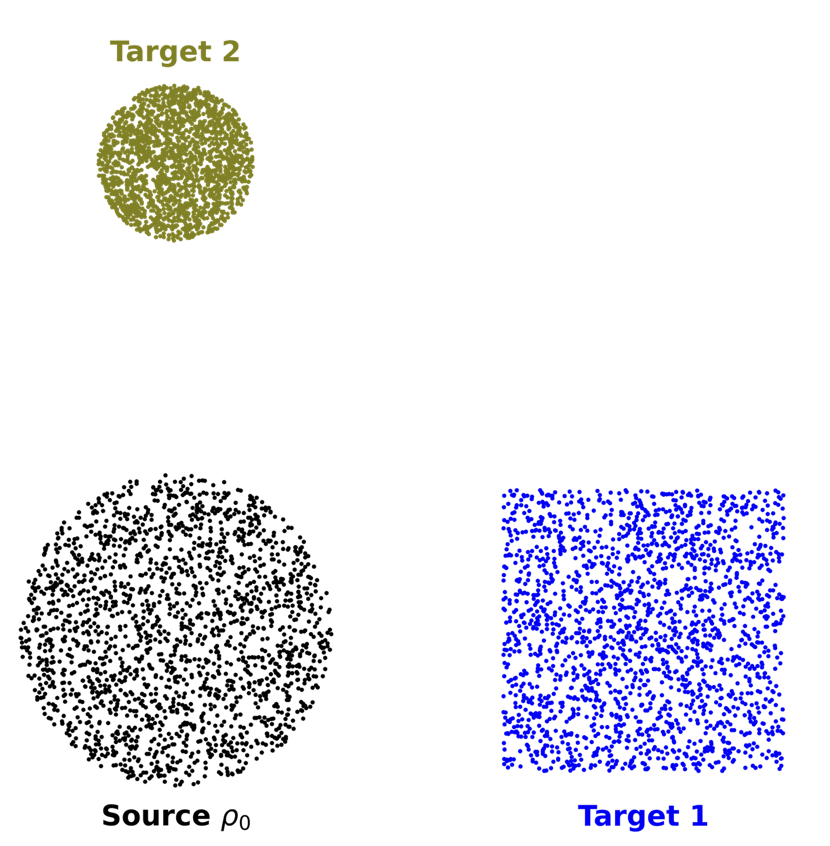}
\caption{Data}
\label{fig:schematic1}
\end{subfigure}
\begin{subfigure}{0.30\textwidth}
\centering
\includegraphics[scale=0.28]{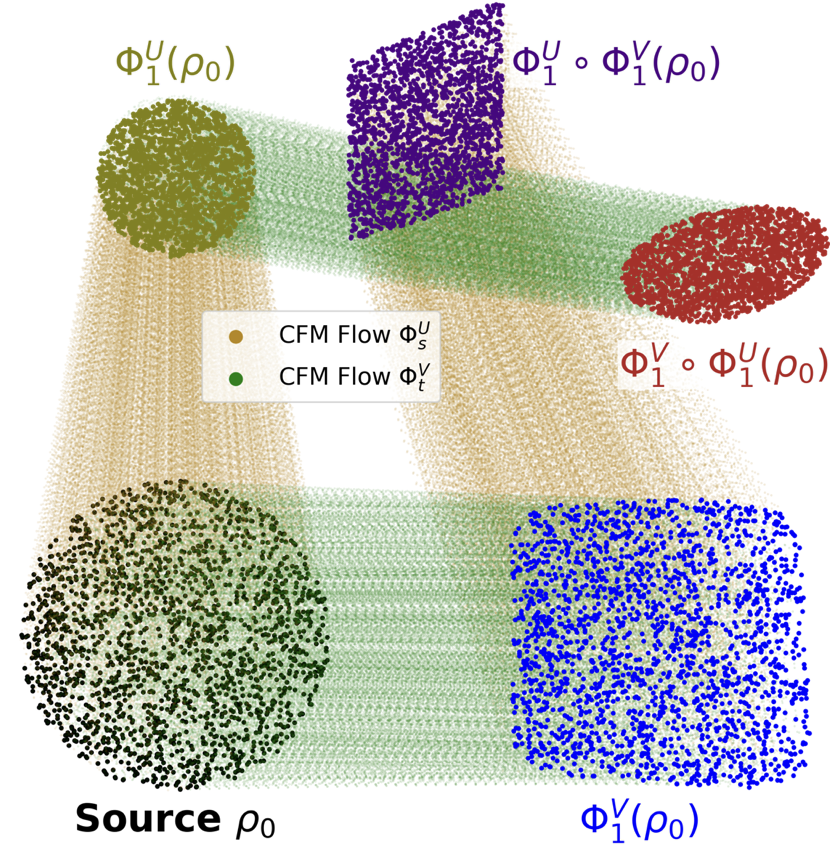}
\caption{CFM}
\label{fig:schematic2}
\end{subfigure}
\begin{subfigure}{0.30\textwidth}
\centering
\includegraphics[scale=0.28]{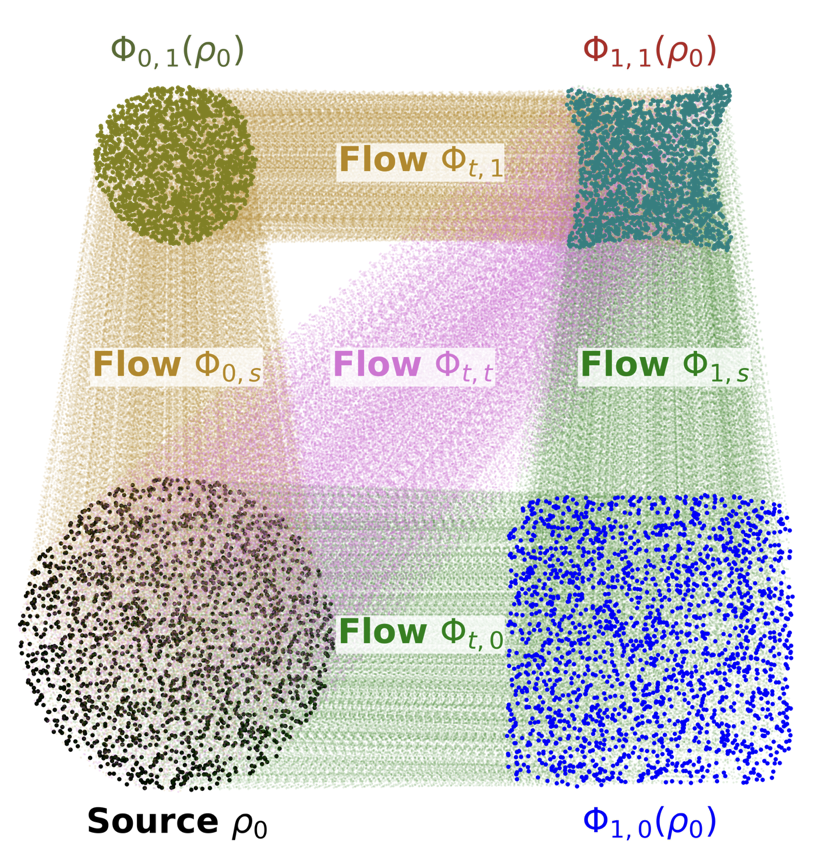}
\caption{PiFM (ours)}
\label{fig:PiFM}
\end{subfigure}
\caption{Learned trajectories for multi-marginal transport between a source (unit disc) and two targets (square, small disc). Traditional CFMs do not commute by composing flows of two separate processes, whereas PiFM (ours) is able to model commutative multi-process trajectories.}
\end{figure*}

To address this limitation, we introduce \emph{Path-independent Flow Matching (PiFM)}, a framework for learning multi-parameter generative dynamics driven by vector fields with a path-independent transport of distributions property. PiFM models probability paths over multiple temporal domains and enforces path independence through structural conditions on the vector fields, echoing classical commutativity and integrability criteria for flows \citep{lee2003introduction, warner1983foundations, abraham1984}. This yields a well-defined notion of composing transformations, enabling consistent interpolation, extrapolation, and analogy-like operations in distribution space. We further establish an empirical connection between PiFM and Wasserstein barycenters \citep{agueh2011barycenters, bigot2017characterization}, with additional conditions for exact matching. Beyond the associated Wasserstein simplex parameter space, we show that our model offers extrapolation capabilities, extending this framework as a natural generalization of multi-marginal generative modeling \citep{cao2019multimarginal,albergo2023multimarginal, Beier2022, pass2015multimarginal}. 
As a result, PiFM is shown to support OOD inference while avoiding the need for full manipulation of densities in distribution space as in methods such as \cite{atanackovic2025meta,haviv2411wasserstein}. Lastly, we evaluate PiFM through several computational experiments ranging from image translation and adaptive target generation to synthetic datasets for demonstration purposes.

Our contributions are summarized as follows:
\vspace{-0.3em}
\begin{enumerate}[leftmargin=14pt]
\item \textbf{Path-independent Flow Matching (PiFM):} We introduce PiFM, a multi-parameter Flow Matching framework driven by multiple vector fields. We establish the existence of path-independent flows and derive the sufficient integrability conditions ensuring consistency.
\vspace{-0.3em}
\item \textbf{Conditional-to-Marginal Flow Consistency:} We show how conditional probability paths induce coherent marginal dynamics in the multi-parameter setting. This leads to a training objective that combines vector field regression with a path-independence regularizer, which we further show becomes unnecessary in the affine case.
\vspace{-0.3em}
\item \textbf{Connection to Wasserstein Barycenters:} We establish a connection between PiFM and Wasserstein barycenters, enabling both interpolation and extrapolation of distributions.
\vspace{-0.3em}
\item\textbf{Empirical Validation:} We demonstrate the effectiveness of PiFM on synthetic and image translation tasks. PiFM outperforms prior methods such as Meta Flow Matching~\citep{atanackovic2025meta} and Curly Flow Matching~\citep{petrovic2025curly} in low-dimensional domain shift settings, and achieves improved compositionality and consistency in high-dimensional multi-attribute image translation. We further validate PiFM on a dense single-cell reprogramming dataset~\citep{schiebinger2019optimal}, treating experimental day and pluripotency score as two interacting flows to test path-independent generation in a high-dimensional, multi-stage biological process.

\end{enumerate}
\vspace{-0.7em}
The paper is organized as follows. Section~\ref{sec:background} reviews the necessary background on flow matching. Section~\ref{sec:method} introduces the PiFM framework, its theoretical properties, and its connection to Wasserstein barycenters. Section~\ref{sec:experiments} presents experimental results in both low-dimensional synthetic domain-shift settings and high-dimensional image translation tasks and single-cell application.

\vspace{-0.8em}

\section{Background: Flow Matching and Minibatch Optimal Transport}\label{sec:background}

\vspace{-0.8em}

The primary objective in distribution matching is to recover a transport map $\psi: \mathbb{R}^d \to \mathbb{R}^d$ that transports a source density $q(a)$ to a target density $q(b)$. This push-forward requirement, denoted $q(b) = [\psi]_{\#}(q(a))$, underpins both optimal transport and modern generative modeling. In the empirical setting, where we lack analytical access to densities and observe only finite datasets $\{x_0^i\}_{i=1}^N \sim q(a)$ and $\{x_1^j\}_{j=1}^M \sim q(b)$, the challenge lies in approximating $\psi$ from discrete samples. A powerful approach to model $\psi$ is through dynamical systems, specifically ordinary differential equations (ODEs), which motivate the Flow Matching approach. 

Flow Matching \citep{liu2022, lipman2023, tong2023improving, albergo2023} provides a framework for learning a dynamical system by establishing a correspondence between a probability density path $p: [0,1] \times \mathbb{R}^d \to \mathbb{R}_{\geq 0}$ and a vector field $v:[0,1] \times \mathbb{R}^d \to \mathbb{R}^d$. We say that $p_t$ is generated by $v_t$ if the pair $(v,p)$ satisfies the following \emph{continuity equation}
\begin{equation}\label{eq:continuity}
\frac{\partial p_t(x)}{\partial t} + \nabla \cdot (p_t(x) v_t(x)) = 0,
\end{equation}
subject to the boundary condition $p_0=q(a)$ and $p_1=q(b)$.

While solving for the marginal vector field directly is intractable, \emph{Conditional Flow Matching} (CFM) circumvents this by exploiting the linearity of \eqref{eq:continuity} and considering the marginal densities $p_t(x)$ as a mixture of simpler conditional densities $p_t(x|z)$ conditioned on a latent variable $z$. In the CFM setting, the conditional densities $p_t(x|z)$ are induced by a conditional vector field $u_t(x|z)$. Consider a latent variable $z$ with corresponding density $q$. Then, $p_t(x) = \int p_t(x | z) q(z)\, dz,$
and by defining the unconditional vector field  $u_t(x)$ by $u_t(x) := \mathbb{E}_{q(z)} \left[  \frac{u_t(x | z)p_t(x | z)}{p_t(x)} \right]$, \citep{lipman2023} showed that when the conditional vector field $u_t(x|z)$ generates the conditional probability path $p_{t}(x|z)$, the unconditional vector field $u_{t}(x)$ generates the marginal probability path $p_t(x)$.

 For CFM, a vector field $v^{\theta}(t,x)$ is trained to approximate the intractable unconditional vector field by regressing onto the conditional fields. The CFM learning objective is to find $\theta$ that minimizes the loss function $\mathcal{L}_{\mathrm{CFM}}(\theta)$ given by 
\begin{equation*}
\mathcal{L}_{\mathrm{CFM}}(\theta)
=
 \mathbb{E}_{t \sim U(0,1),z\sim q,\, x\sim p_t(\cdot\mid z)}
\left[\left\| v_\theta(t,x) - u_t(x| z) \right\|^2 \right] .
\end{equation*}
This formulation enables simulation-free training of continuous normalizing flows by sampling $t\sim\mathcal{U}(0,1)$, $z\sim q(z)$, and $x\sim p_t(\cdot| z)$, while relying on efficient sampling from  $q(z)$.

The density $q(z)$ can be specified in different ways. In independent CFM (I-CFM), $q(z)=q(a)q(b)$, while in optimal transport CFM (OT-CFM), $q(z)=\pi(a,b)$, where $\pi$ is an optimal transport coupling between $q(a)$ and $q(b)$. Training under I-CFM uses independently sampled pairs $(a,b)$, whereas OT-CFM requires sampling from $\pi$, typically approximated via mini-batch optimal transport due to the cubic time and quadratic memory cost of full OT \citep{Cuturi2013SinkhornDL}. Despite being approximate, mini-batch OT preserves the transport geometry in expectation and performs well empirically \citep{pmlr-v139-fatras21a}; the exact OT-CFM objective is recovered when the batch size equals the full dataset.

\vspace{-0.8em}

\section{Multi-parameter Flow Matching and Path Independence}\label{sec:method}

\vspace{-0.8em}

We begin by introducing a multi-parameter flow matching formulation. To simplify our notations, we restrict our exposition to the case of two independent parameters, denoted by $t$ and $s$ (see Section \ref{sec:generalization} for details on the more general case). Then, we introduce the PiFM framework, which extends CFM to a multi-parameter setting. We present the modeling choices and assumptions that form the foundation of PiFM, and show how these elements collectively ensure coherent marginal generation and path-independent transport of distributions. Finally, we describe how to train PiFM using a deep learning approach and establish a connection to the Wasserstein Barycenter. Proofs are included in the Appendix~\ref{sec:Proofs_of_theorems} and Appendix~\ref{sec:appendix_wb_relation}.

\vspace{-0.4em}

\subsection{Multi-parameter Flows and Path Independence}\label{sec:3.1}

\vspace{-0.4em}

Suppose we have two continuously differentiable vector fields $u,v:[0,1]^{2}\times \mathbb{R}^{d}\to \mathbb{R}^{d}$ depending on parameters $t$ and $s$ with associated ODEs:
\begin{equation}
    \frac{\partial \phi}{\partial t}(t, s)=u_{t, s}(\phi(t, s)), \hspace{0.5em} \forall s\in [0,1] \hspace{0.5em};\hspace{0.5em} \frac{\partial \psi}{\partial s}(t, s)=v_{t, s}(\psi(t, s)), \hspace{0.5em} \forall t\in [0,1]. \label{eq:phi_ts}
\end{equation}
With initial conditions $\phi_{0,0}(x)=x$ and $\psi_{0,0}(x)=x$, both equations are well-defined at $(t, s)=(0,0)$, yielding the  functions $\phi_{t,0}$ and $\psi_{0,s}$. These solutions can then be used to define consistent initial conditions along the axes, namely $\phi_{0,s}(x)=\psi_{0,s}(x)$ and $\psi_{t,0}(x)=\phi_{t,0}(x)$, which in turn determine the full solutions $\phi_{t, s}$ and $\psi_{t, s}$. Note that this construction proceeds by solving ODEs sequentially, existence and uniqueness at each stage are guaranteed by standard ODE theory (see, e.g., \citep{wiggins2003introduction}). Associated with these solutions, we define the flow maps $\Phi,\Psi:[0,1]^{2}\times \mathbb{R}^{d}\to \mathbb{R}^{d}$ by $\Phi_{t, s}(x)=\phi_{t, s}(x)$ and $\Psi_{t, s}(x)=\psi_{t, s}(x)$.

Analogous to the one-parameter setting, the vector fields $u_{t, s}$ and $v_{t, s}$ generate densities $p_{t, s}$ and $q_{t, s}$, respectively, provided that the following continuity equations hold for all $t,s\in [0,1]$:
%\begin{align}
 %   \frac{\partial p_{t, s}(x)}{\partial t} + \nabla \cdot \big(p_{t, s}(x)\,u_{t, s}(x)\big) &= 0, \hspace{1em} \forall s \in [0,1], \label{eq:distinct-generation1} \\
  %  \frac{\partial q_{t, s}(x)}{\partial s} + \nabla \cdot \big(q_{t, s}(x)\,v_{t, s}(x)\big) &= 0, \hspace{1em} \forall t \in [0,1].\label{eq:distinct-generation2}
%\end{align}
\begin{align}
    \frac{\partial p_{t, s}(x)}{\partial t} + \nabla \cdot \big(p_{t, s}(x)\,u_{t, s}(x)\big) &= 0 \qquad
    \frac{\partial q_{t, s}(x)}{\partial s} + \nabla \cdot \big(q_{t, s}(x)\,v_{t, s}(x)\big) = 0.\label{eq:distinct-generation}
\end{align}
Similarly to the definition of initial conditions for ODEs, we prescribe boundary conditions along the coordinate axes. Given an initial density $\rho_1$ and target densities $\rho_2$ and $\rho_3$, we define the boundary conditions by setting $p_{0,s}=(\Psi_{0,s})_{\#}\rho_1$, $p_{1,s}=(\Psi_{1,s})_{\#}\rho_2$, $q_{t,0}=(\Phi_{t,0})_{\#}\rho_1$ and $q_{t,1}=(\Phi_{t,1})_{\#}\rho_3$. With these constructions, the continuity equations yield consistent transport of $p_{t, s}$ along the flow $\Phi_{t, s}$ and of $q_{t, s}$ along the flow $\Psi_{t, s}$.

We say the transport flow is distributionally path-independent (or commutative) if 
\begin{equation}\label{eq:path_independence}
    (\Phi_{t, s})_{\#} \circ (\Psi_{0,s})_{\#} p_{0,0}
    =
    (\Psi_{t, s})_{\#} \circ (\Phi_{t,0})_{\#} p_{0,0}, \hspace{1em} \forall t,s\in [0,1].
\end{equation}
Intuitively, \eqref{eq:path_independence} states that moving first along the \(t\)-direction while fixing $s=0$ using the flow $\Phi_{t, s}$, and then along the $s$-direction using $\Psi_{t, s}$, yields the same result as performing the operations in the reverse order. The implication this has for path independence  is that if we consider a parametrized path $\gamma(l):[0,1]\to [0,1]^2$, with $\gamma(l)=(\gamma^{(1)}(l),\gamma^{(2)}(l))=(t(l),s(l))$, then \eqref{eq:path_independence} implies that the order in which we apply the transport maps $\Phi_{\#}$ and $\Psi_{\#}$ does not matter: we always arrive at the same distribution at the point $(t, s)$ along the trajectory prescribed by $\gamma$. This is formalized via
\begin{equation*}
    (\Phi_{\gamma(l)})_{\#} \circ (\Psi_{0,\gamma^{(2)}(l)})_{\#} p_{0,0}
    =
    (\Psi_{\gamma(l)})_{\#} \circ (\Phi_{\gamma^{(1)}(l),0})_{\#} p_{0,0}.
\end{equation*}
To preserve distributional path-independence, it is sufficient if there exists a single probability density path that satisfies continuity equations \eqref{eq:distinct-generation}, that is, when $p_{t, s}$ and $q_{t, s}$ coincide for all $t$ and $s$ in $[0,1]$. We formally state this as a lemma.

\begin{lemma}[Path-independent transport of distributions]\label{lem:path}
    Consider $u_{t, s}$ and $v_{t, s}$ vector fields. If there exists a unique generated probability density path $p_{t, s}$ satisfying \eqref{eq:distinct-generation}, then the associated flows $\Phi_{t, s}$ and $\Psi_{t, s}$ satisfy \eqref{eq:path_independence}. 
\end{lemma}

Lemma \ref{lem:path} provides a sufficient condition for path-independent transport of distributions. However, other sufficient conditions for this property exist. One way this can occur is when the system defined by \eqref{eq:phi_ts} admits a unique joint trajectory $\phi_{t, s}$. This implies pointwise path independence, while for our formulation we only need distributional path independence, which is a weaker property. A necessary and sufficient condition for a unique solution $\phi_{t, s}$ to exist is provided in  Proposition~19.29 of \citep{lee2003introduction} which uses the partial differential equation described in \eqref{eq:integrability}. 

\begin{proposition}[(Pointwise) path-independence equation]\label{prop:alternative_path}
    Suppose that the vector fields $u_{t, s}$ and $v_{t, s}$ are satisfy
    \begin{equation}\label{eq:integrability}
       \partial_s u_{t, s} -\partial_t v_{t, s} = \left[u_{t, s},v_{t, s}\right], \quad \forall x\in \mathbb{R}^d
    \end{equation}
  where $\partial_s u_{t, s}:={\partial u_{t, s}(x)}/{\partial s}$,  $\partial_t v_{t, s}:={\partial v_{t, s}(x)}/{\partial t}$, and $[\cdot,\cdot]$ is the Lie Bracket, $\left[u_{t, s},v_{t, s}\right]:=\left(\nabla_x v_{t,s}(x)\right)u_{t,s}(x) -\left(\nabla_x u_{t,s}(x)\right)v_{t,s}(x)$. Then there exists a unique flow $\Phi_{t,s}$ which satisfies \eqref{eq:path_independence}.
\end{proposition}

The preceding proposition can be read as follows, pointwise path-independence can be characterized by the integrability equation \eqref{eq:integrability}. When this equation is satisfied, existence of a joint flow $\Phi_{t,s}$ of $u_{t,s}$ and $v_{t,s}$ is guaranteed. This function satisfies \eqref{eq:path_independence}, so in this case distributional path-independence is also guaranteed.

\vspace{-0.4em}

\subsection{Consistent and Path-independent Transport of Distributions in PiFM}\label{sec:cond_to_marg}

\vspace{-0.4em}

We consider data distributions $q(a)$, $q(b)$ and $q(c)$, and the joint density $q(z)$ with $z=(a,b,c)$. In contrast to the traditional flow matching framework, we consider the probability density path $p:[0,1]^{2}\times \mathbb{R}^{d}\to \mathbb{R}_{\geq 0}$ of two-parameter with boundry conditions $p_{0,0}(x)=q(a)$, $p_{1,0}(x)=q(b)$ and $p_{0,1}(x)=q(c)$. Analogous to CFM, we condition on $z$ yielding,
\begin{equation}
    p_{t, s}(x)=\int p_{t, s}(x|z)q(z)dz,  \label{eq:cond_prob}
\end{equation}
where the conditional density path $p_{t, s}(x|z)$ is induced by some conditional vector fields $u_{t, s}(x|z)$ and $v_{t, s}(x|z)$. We seek two vector fields $u_{t, s},v_{t, s}:[0,1]^{2}\times \mathbb{R}^{d}\to \mathbb{R}^{d}$ that generate $p_{t,s}$ by satisfying the continuity equations for all $t,s\in [0,1]$
%\begin{align}
 %   \frac{\partial p_{t, s}(x)}{\partial t} + \nabla \cdot \big(p_{t, s}(x)\,u_{t, s}(x)\big) &= 0, \hspace{0.5em} \forall s \in [0,1], \label{eq:equal-generation1} \\
  %  \frac{\partial p_{t, s}(x)}{\partial s} + \nabla \cdot \big(p_{t, s}(x)\,v_{t, s}(x)\big) &= 0, \hspace{0.5em} \forall t \in [0,1].\label{eq:equal-generation2}
%\end{align}
\begin{align}
    \frac{\partial p_{t, s}(x)}{\partial t} + \nabla \cdot \big(p_{t, s}(x)\,u_{t, s}(x)\big) &= 0, \qquad
    \frac{\partial p_{t, s}(x)}{\partial s} + \nabla \cdot \big(p_{t, s}(x)\,v_{t, s}(x)\big) = 0.\label{eq:equal-generation}
\end{align}
If the two equations above hold, then the corresponding flows satisfy the distributional path-independence property. Given conditional vector fields $u_{t, s}(x|z)$ and $v_{t, s}(x|z)$, we model the unconditional vector fields by
\begin{equation}
    u_{t, s}(x)=\mathbb{E}_{q(z)}\left[\frac{u_{t, s}(x\mid z)p_{t, s}(x\mid z)}{p_{t, s}(x)} \right] \hspace{0.5em},\hspace{0.5em} v_{t, s}(x)=\mathbb{E}_{q(z)}\left[\frac{v_{t, s}(x\mid z)p_{t, s}(x\mid z)}{p_{t, s}(x)} \right]. \label{eq:u_ts_v_ts}
\end{equation}
We note that the precise forms of the conditional vector fields $u_{t, s}(x|z)$ and $v_{t, s}(x|z)$ are not prescribed a priori and left as a user-defined choice. 

\begin{theorem}\label{thm:gen}
    Let the conditional vector fields $u_{t, s}(x|z)$ and $v_{t, s}(x|z)$ generate the unique conditional probability density path  $p_{t, s}(x|z)$ satisfying the analogous equations \eqref{eq:equal-generation}. Then, the vector fields $u_{t, s}(x)$ and $v_{t, s}(x)$ in \eqref{eq:u_ts_v_ts} generate the unique probability density path $p_{t, s}(x)$ satisfying \eqref{eq:equal-generation}.
\end{theorem}

Note that Theorem \ref{thm:gen} is a statement about how $p_{t,s}$ is consistently transported using vector fields $u_{t,s}$ and $v_{t,s}$, it is not a statement about pointwise or distributional-path-independence. Additionally, Theorem \ref{thm:gen} can be viewed as a generalization of the conditional-to-marginal generation result in CFM, extending the framework to accommodate multiple parameter settings and additional target distributions. As a corollary to Theorem \ref{thm:gen} and Lemma \ref{lem:path}, we have the following result: the flows of the unconditional vector fields have the distributional path-independence property.
\begin{corollary}[Path-independent transport of distributions in PiFM]\label{thm:path_independence}
    Suppose that the vector fields $u_{t, s}(x|z)$ and $v_{t, s}(x|z)$ generate the unique conditional probability density path  $p_{t, s}(x|z)$ satisfying the analogous equations \eqref{eq:equal-generation}. Then, the associated flows $\Phi_{t, s}$ and $\Psi_{t, s}$ of the vector fields $u_{t, s}(x)$ and $v_{t, s}(x)$ given in \eqref{eq:u_ts_v_ts} satisfy \eqref{eq:path_independence}.
\end{corollary}

\vspace{-0.4em}

\subsection{Training via Path-independent Flow Matching}\label{sec:training}

\vspace{-0.4em}

Now we discuss a neural approach for training a PiFM model. Given two parameterized vector fields $u_{t, s}^{\theta}$ and $v_{t, s}^{\theta}$, PiFM learns the dynamics by regressing these fields onto the corresponding conditional vector fields. We define the flow-matching regression loss as
\begin{equation*}
\mathcal{L}_{\text{FM}}(\theta)
    =
    \mathbb{E}_{(t,s) \sim U[0,1]^{2},z\sim q,\, x\sim p_{t,s}(\cdot\mid z)}\Big[
    \|u_{t, s}^{\theta}(x)-u_{t, s}(x\mid z)\|^2 
    +
    \|v_{t, s}^{\theta}(x)-v_{t, s}(x\mid z)\|^2
    \Big]
\end{equation*}
For conditional vector fields that do not automatically satisfy the pointwise path-independence condition in \eqref{eq:integrability}, we consider additionally penalizing violations of this condition, this is particularly useful in settings such as integrating PiFM with Curly Flow Matching~\citep{petrovic2025curly}; see Section~\ref{sec:Inference}. The path-independence regularizer is
\begin{equation*}
\begin{aligned}
    \mathcal{L}_{\text{Pi}}(\theta)
    =
    \mathbb{E}\Big[
    \Big\|
    &\partial_s u_{t,s}^{\theta}(x\mid z)
    - \partial_t v_{t,s}^{\theta}(x\mid z) - [u_{t,s}^{\theta}(x\mid z),v_{t,s}^{\theta}(x\mid z)] %\\
    %&+ \big(\nabla_x u_{t,s}^{\theta}(x)\big)v_{t,s}^{\theta}(x) \\
    %&- \big(\nabla_x v_{t,s}^{\theta}(x)\big)u_{t,s}^{\theta}(x)
    \Big\|^2
    \Big],
\end{aligned}
\end{equation*}
with the expectation computed over $(t,s)$, $z$ and $x$ as in the Flow Matching regression loss. Given a regularization weight $\lambda \geq 0$, the PiFM objective is $\mathcal{L}_{\text{PiFM}}(\theta;\lambda)=\mathcal{L}_{\text{FM}}(\theta)+\lambda \mathcal{L}_{\text{Pi}}(\theta)$. Training proceeds analogously to CFM. We sample $(t,s)\sim\mathcal{U}[0,1]^2$ and $z=(a,b,c)\sim q(z)$, where $q(z)$ may be specified using a variety of couplings such as the coupling that results from using OT couplings in $(a,b)$ 
 and $(a,c)$, multimarginal OT coupling in $(a,b,c)$ or independent coupling which arises from independently sampling in each of the three distributions separately. We then sample $x\sim p_{t,s}(\cdot\mid z)$ and update $\theta$ by minimizing $\mathcal{L}_{\text{PiFM}}(\theta;\lambda)$. The full training procedure is summarized in Algorithm~\ref{alg:pfm}.

One specific choice of conditional vector fields and conditional probability path is given by

\begin{equation}
    u_{t, s}(x \mid z) = b - a, \hspace{1em} v_{t, s}(x \mid z) = c - a,\hspace{1em} p_{t, s}(x \mid z) = \mathcal{N}\big(x \mid \mu_{t,s}(z), \sigma^{2}\big), \label{eq:conditionals}
\end{equation}

\begin{wrapfigure}{r}{0.5\textwidth}
\vspace{-2em} % adjust if needed
\begin{minipage}{0.48\textwidth}
\begin{algorithm}[H]
\caption{Path-Independent Flow Matching (PiFM)}
\label{alg:pfm}
\begin{algorithmic}[1]
\INPUT Efficiently samplable $q(a,b,c)$, $p_{t, s}(x | (a,b,c))$, computable conditional vector fields $u_{t, s}(x | (a,b,c))$ and $v_{t, s}(x | (a,b,c))$, initial two-headed network $(u^{\theta}, v^{\theta})$ with a single backbone, learning rate $\eta > 0$ and regularization $\lambda>0$.
\WHILE{training}
    \STATE Sample $(a,b,c) \sim q(a,b,c)$, $(t, s) \sim \mathcal{U}([0,1]^2)$, $x \sim p_{t, s}(x | (a,b,c))$
    \STATE Compute loss $\mathcal{L}_{\text{PiFM}}(\theta;\lambda)$
    \STATE Update $\theta \leftarrow \theta - \eta \nabla_\theta \mathcal{L}_{\mathrm{PiFM}}(\theta)$
\ENDWHILE
\STATE \textbf{return} $u_\theta, v_\theta$
\end{algorithmic}
\end{algorithm}
\end{minipage}
\vspace{-1.0em}
\end{wrapfigure}
for fixed $\sigma > 0$. This generalizes the affine interpolation commonly used in CFM to a multi-parameter setting by introducing a two-dimensional interpolation of the form $\mu_{t,s}(z)=a + (b-a)t + (c-a)s$. This choice of conditional vector fields and conditional probability density paths is guaranteed to induce unconditional vector fields that transport the marginal probability density path $p_{t,s}$ in a consistent and path-independent manner by satisfying both Theorem~\ref{thm:gen} and Corollary~\ref{thm:path_independence}. We emphasize that $\mathcal{L}_{Pi}$
 can promote both pointwise and distributional path-independence. It directly affects pointwise path-independence through the integrability condition \eqref{eq:integrability}, while distributional path-independence follows from the pointwise property. Therefore, the path-independence regularizer is unnecessary when the assumptions of Corollary \ref{thm:path_independence} hold, but may be beneficial when they are violated, as demonstrated in Section \ref{sec:Inference}.

\vspace{-0.4em}

\subsection{Relation to Wasserstein Barycenter}\label{sec:W_barycenter}
%\kaly{This section is too long, we can shorten it a bit!}
\vspace{-0.4em}

In the previous subsections, we discussed PiFM, its properties, and its training. Now, we show that PiFM empirically approximates the Wasserstein Barycenter between three distributions. Additionally, under suitable conditions, we provide theoretical guarantees that PiFM matches the Wasserstein barycenter, demonstrating that PiFM recovers a geometrically meaningful notion that is globally consistent with interpolation between distributions.

Let $\lambda:=(\lambda_1,\ldots,\lambda_K)$ with $\lambda_j\geq 0$ and $\sum_{j=1}^{K}\lambda_j=1$ for some $K\in \mathbb{N}$. The Wasserstein barycenter $\rho_\lambda$ of probability measures  $\rho_1,\ldots,\rho_K \in \mathcal{P}_2(\mathbb{R}^d)$ with finite second moment is a distribution that minimizes the multi-marginal optimal transport cost $J(\gamma,\lambda)$ given below 
\begin{equation*}
\label{eq:marg_OT}
    J(\gamma,\lambda)
    =
    \int
    \sum_{j=1}^{K}
    \lambda_j
    \left\|
    T_\lambda(x_1,\ldots,x_K)-x_j
    \right\|^2
    d\gamma,
\end{equation*}
over all couplings $\gamma$ of $\rho_1,\ldots,\rho_K$ with $T_\lambda(x_1,\ldots,x_K)=\sum_{j=1}^{K}\lambda_j x_j$.

One popular algorithm for computing the Wasserstein barycenter is the free support computation algorithm \citep{cuturi2014a}, available in the POT package \citep{POT}. Empirically, we observe in Figure \ref{fig:PFM_vs_WB_extended}a. that distributions $\rho^{\text{PiFM}}_{t, s}$ computed from our method coincide with the Wasserstein barycenter  $WB_{1-t-s, t, s}(\rho_0, \rho_1, \rho_2)$ obtained with the free support algorithm. In addition, we empirically observe that our framework extrapolates beyond the simplex constraint on $\lambda$ in Figure \ref{fig:PFM_vs_WB_extended}b. Specifically, PiFM and free-support generated distributions coincide in the extrapolation region.

To support our claim that PiFM can approximate the Wasserstein barycenter, we provide additional experimental results in Supplementary Figure~\ref{fig:PFM_vs_WB_csc} on more complex distributions. In addition, using the multi-marginal optimal transport formulation and the
notion of admissible families of deformations from \citep{boissard2013distributionstemplateestimatewasserstein}, Theorem~\ref{theorem:WB_exact_matching} gives sufficient conditions under which the PiFM-generated distribution $\hat{\rho}_{t, s}$  is the minimizer of the Wasserstein barycenter objective. From the discussion in section \ref{sec:cond_to_marg} note that $\hat{\rho}_{t,s}=(\Phi_{t,s})_{\#}\circ (\Psi_{0,s})_{\#}\rho_0=(\Psi_{t,s})_{\#}\circ (\Phi_{t,0})_{\#}\rho_0=:\Gamma_{t,s}(\rho_0)$. 

\begin{theorem}\label{theorem:WB_exact_matching}
Consider $\rho_0,\rho_1,\rho_2 \in \mathcal{P}(\mathbb{R}^d)$ with $\lambda = (1-t-s, t, s) $. Suppose that the optimal transport maps $T_1: \rho_0 \rightarrow \rho_1, \ T_2: \rho_0 \rightarrow \rho_2 $ belong to an admissible family of deformations, and the conditional paths satisfy $z_{t, s}(a) = (1-t-s)a + tT_1(a) + sT_2(a)$. Then, under $\sigma \rightarrow 0 $, PiFM approximates the Wasserstein barycenter between $(\rho_0, \rho_1, \rho_2)$, i.e , $\hat{\rho}_{t, s} = \Gamma_{t,s}(\rho_0)$ satisfies $\hat{\rho}_{t, s} = WB_{\lambda}(\rho_0, \rho_1, \rho_2)$. Additionally, %we have %$ J^{\mathrm{PiFM}}(\rho_0, \lambda) $ minimizes the multi marginal Optimal Transport (1) problem, i.e: 
the optimal transport cost is
\begin{equation*}
    \int \Big[
(1-t-s)\, \|z_{t, s}(a) - a\|^2  
+\, t\, \|z_{t, s}(a) - T_1(a)\|^2 
+\, s\, \|z_{t, s}(a) - T_2(a)\|^2
\Big] \, d\rho_0(a).
\end{equation*}
\end{theorem}
Additional results relating PiFM to the Wasserstein barycenter are presented in Appendix~\ref{sec:appendix_wb_relation} and Appendix~\ref{sec:additional_results}. Since computing the Wasserstein barycenter is computationally expensive, the results in this section highlight the potential of PiFM as a neural network-based alternative for approximating Wasserstein barycenters. Section~\ref{sec:comp_WPOT} provides an initial exploration of the computational time differences, while a more in-depth analysis is left for future work.

\vspace{-1em}
\begin{wrapfigure}{r}{0.5\textwidth}
    \centering
    \begin{minipage}{\linewidth}
    \includegraphics[width=\linewidth]{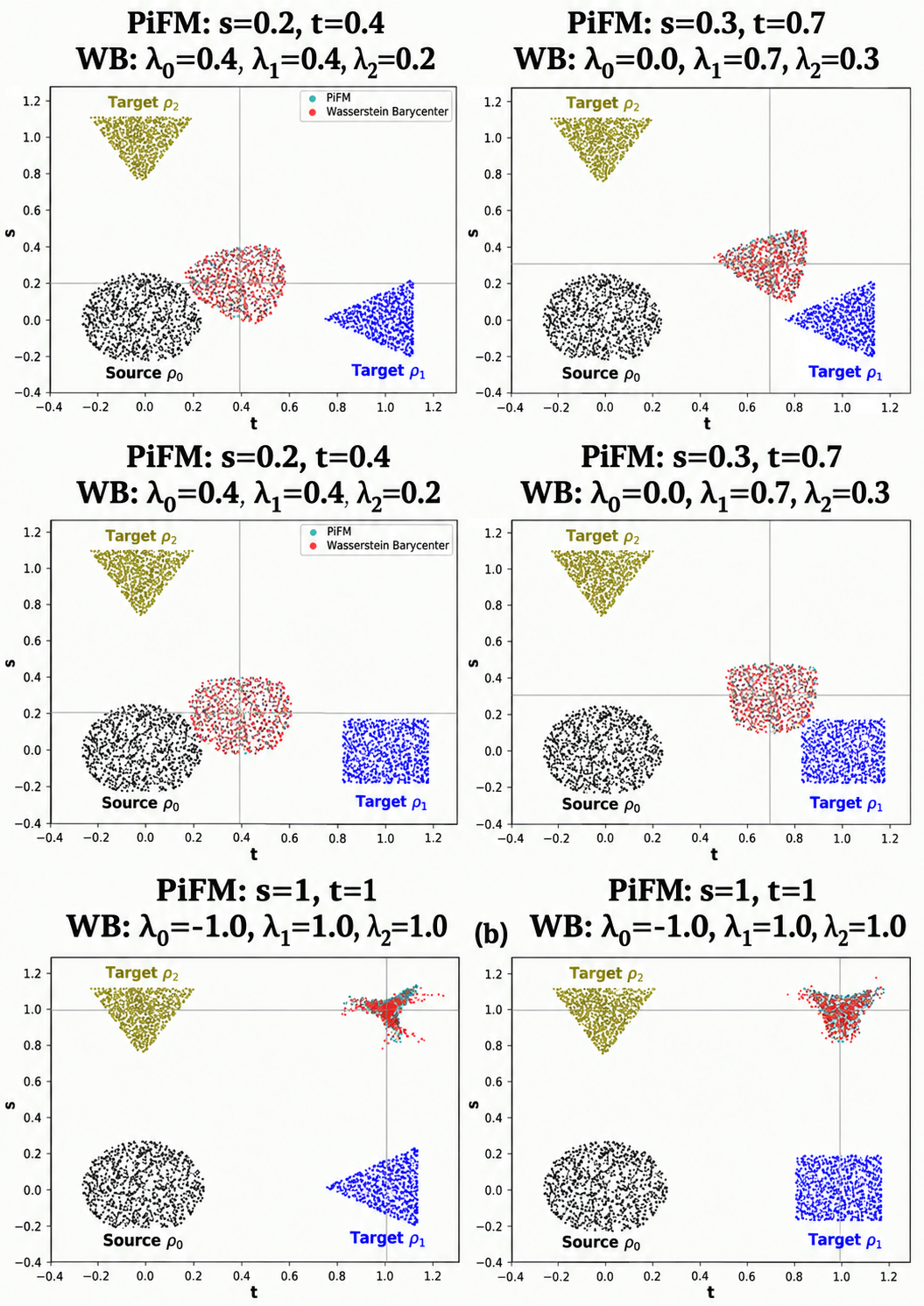}
    \caption{\small Comparison between output distributions from PiFM model (ours) and Wasserstein Barycenters for a) within and  b) outside the probability simplex . The predicted distributions from the PiFM model correspond to the Wasserstein barycenters.}
    \label{fig:PFM_vs_WB_extended}
    \end{minipage}
    \vspace{-4em}
\end{wrapfigure}

\section{Experiments}\label{sec:experiments}

\vspace{-0.8em}

We present empirical results for low- and high-dimensional data settings. In the low-dimensional setting, we demonstrate that PiFM outperforms Meta Flow Matching (MFM) in \citep{atanackovic2025meta} on a toy dataset under domain shift and show how PiFM can be incorporated with Curly Flow Matching \cite{petrovic2025curly}. In the high-dimensional case, we evaluate PiFM's performance on unsupervised image translation and demonstrate its superiority over traditional CFM in data extrapolation. We also evaluate on a real-world single-cell RNA-seq reprogramming dataset~\citep{schiebinger2019optimal}, where PiFM models cellular states while jointly predicting experimental day and pluripotency score, testing whether the learned dynamics capture both chronological progression and fate acquisition toward successful reprogramming. In the following experiments we use OT couplings between $(a,b)$ and between $(a,c)$ for training PiFM.

\vspace{-0.4em}
\subsection{Low Dimensional Setting}\label{sec:4.1}
\vspace{-0.4em}
To assess PiFM's ability to infer samples under out-of-distribution conditions, we use synthetic data in which the distribution is shifted spatially in one direction and changes shape and shifts simultaneously in another direction.

\vspace{-0.4em}

\subsubsection{PiFM for Inferring Unseen Population on Shifted Domain}
\vspace{-0.4em}
{\bf Toy dataset.} We construct unit discs with centers at (0, 0) and (0, 5) as training source distributions and small discs $(r = 0.5)$ with centers at (0, 5) and (5, 5) as training target distributions. Given the unit disc with center at (0, 2.5) as the unseen distribution in the source, the goal is to infer the small disc with center at (5, 2.5) as expected target.

\begin{figure}[t]
\centering

\begin{subfigure}[t]{0.24\textwidth}
\centering
\includegraphics[width=\linewidth]{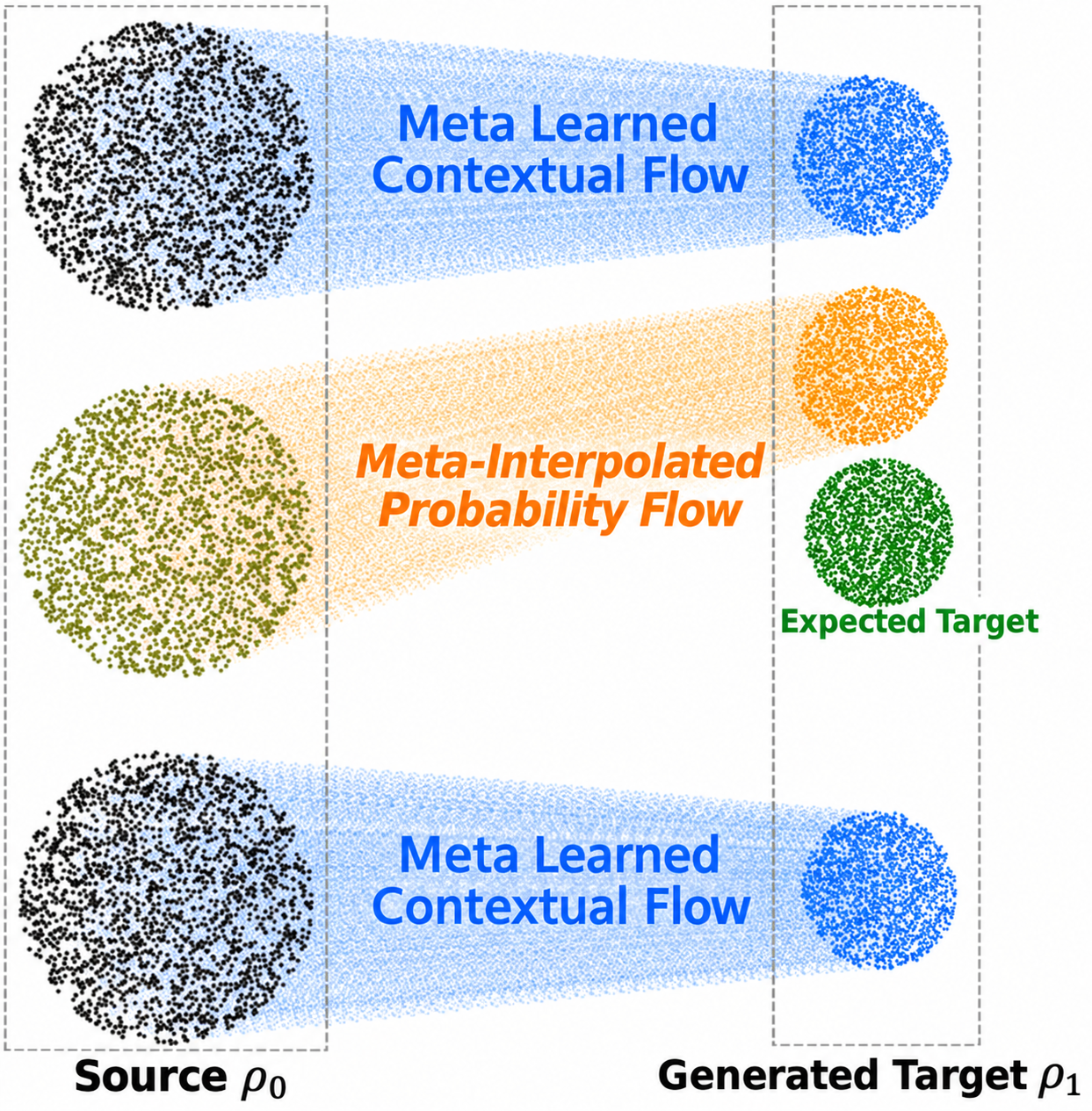}
\caption{MFM}
\label{fig:meta_compare1}
\end{subfigure}
\hfill
\begin{subfigure}[t]{0.24\textwidth}
\centering
\includegraphics[width=\linewidth]{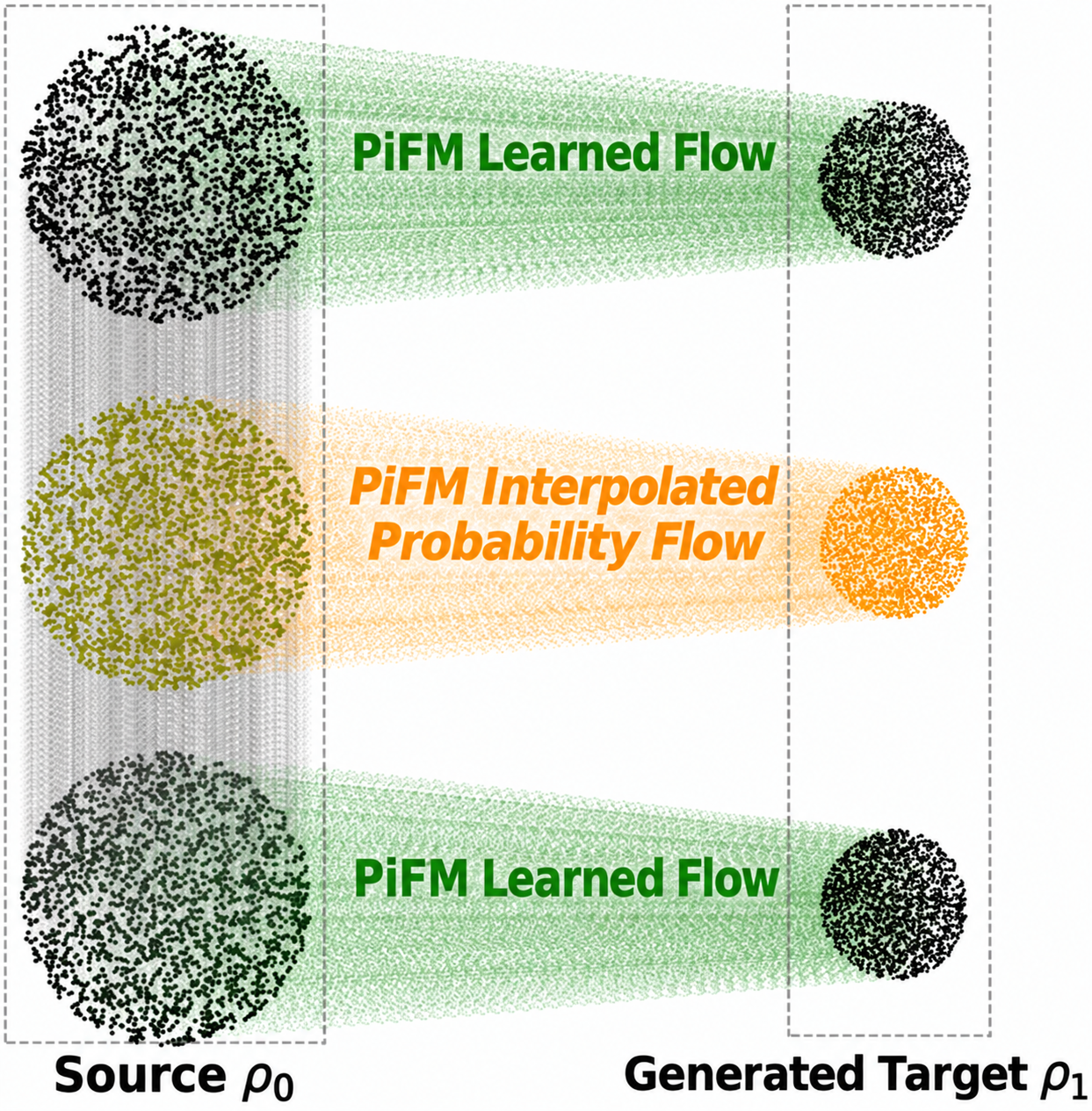}
\caption{PiFM (ours)}
\label{fig:meta_compare2}
\end{subfigure}
\hfill
\begin{subfigure}[t]{0.24\textwidth}
\centering
\includegraphics[width=\linewidth]{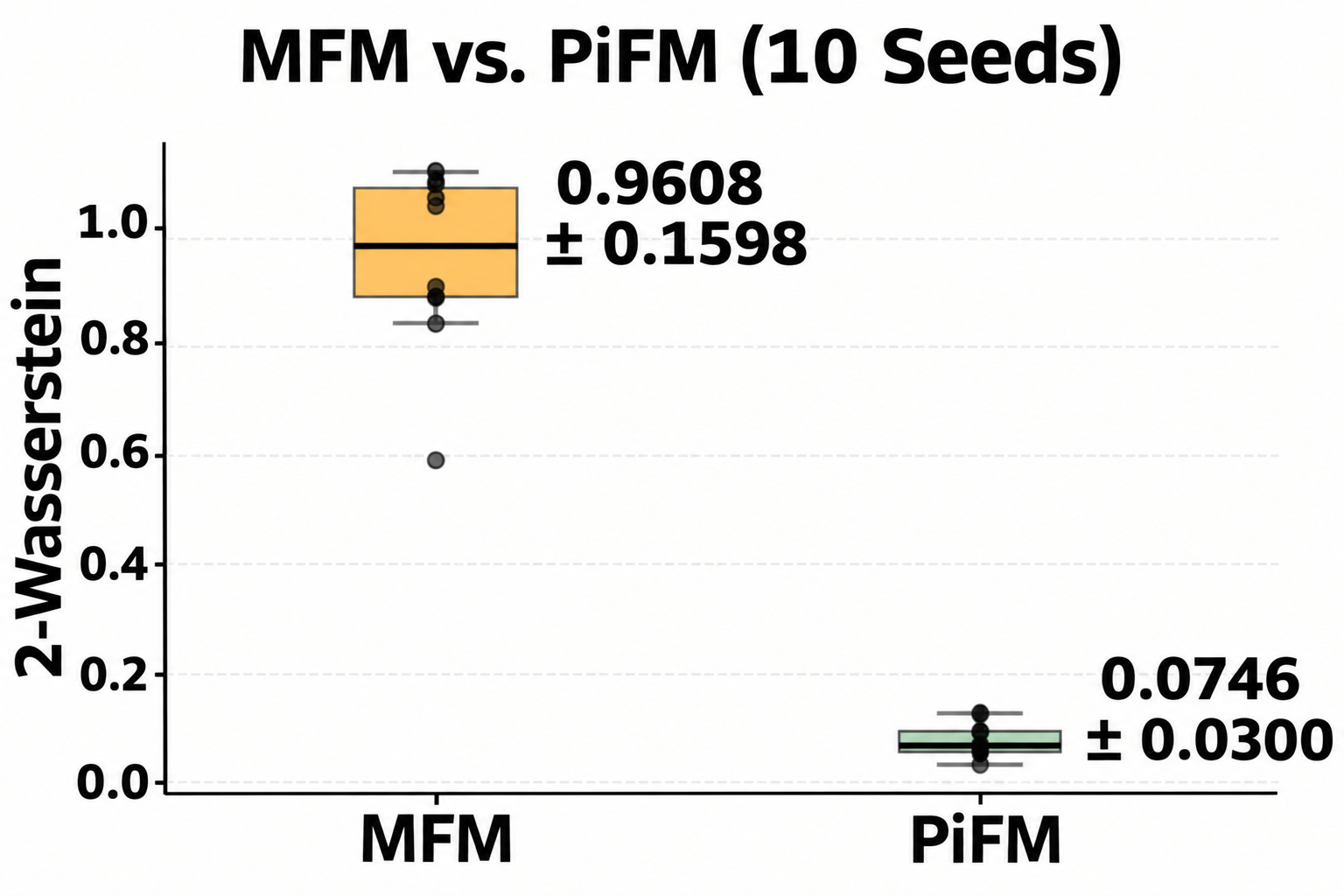}
\caption{Single Batch-size}
\label{fig:seeds_compare}
\end{subfigure}
\hfill
\begin{subfigure}[t]{0.24\textwidth}
\centering
\includegraphics[width=\linewidth]{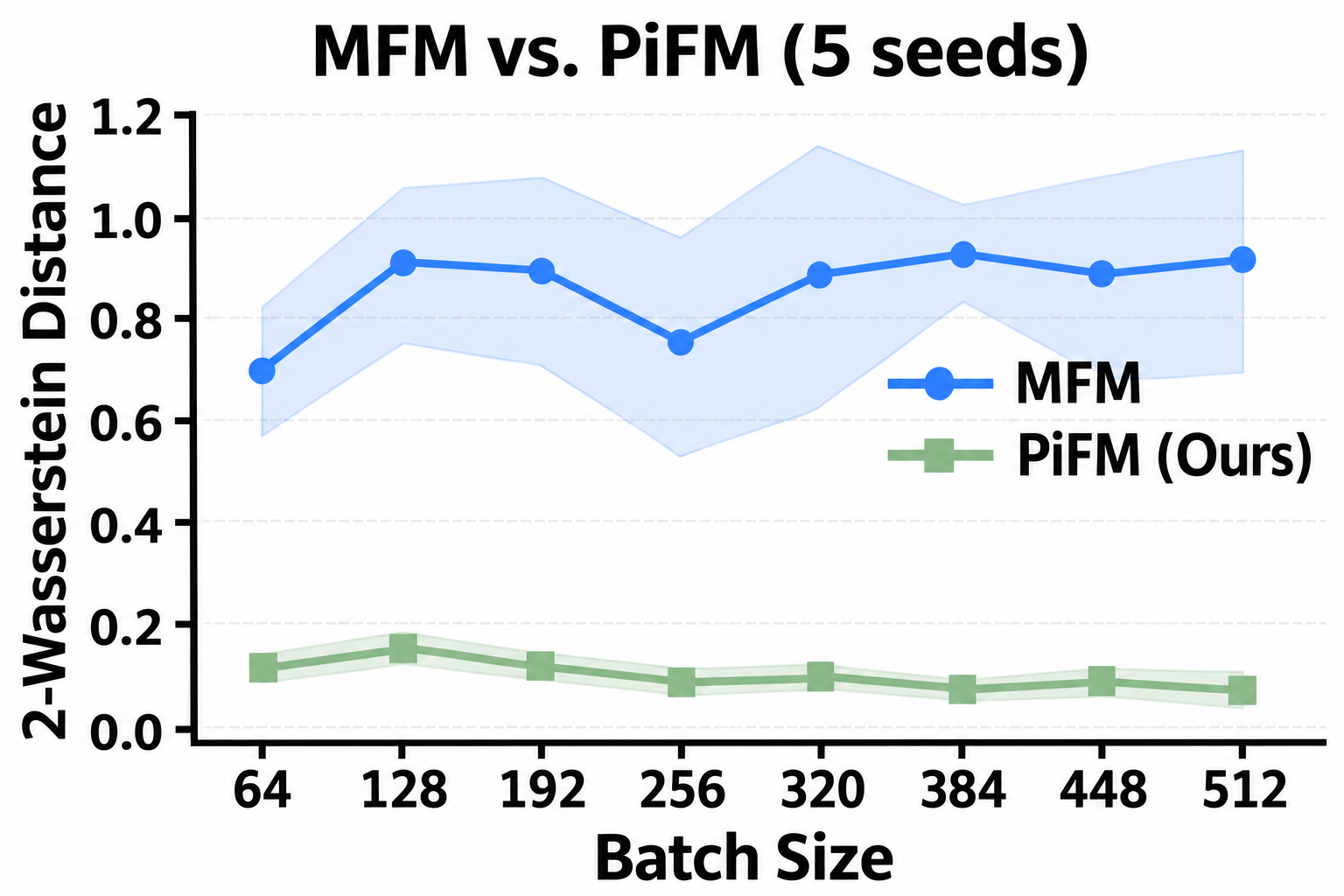}
\caption{Across Batch-size}
\label{fig:batchsize_compare}
\end{subfigure}

\caption{\small Comparison of Meta Flow Matching (MFM) and PiFM on a toy dataset, including qualitative trajectories and quantitative evaluation. Panels (a) and (b) show that MFM generalizes poorly to unseen source embeddings, while PiFM maintains consistent trajectories. Panels (c) and (d) report the 2-Wasserstein distance between ground truth and generated samples, showing that PiFM achieves more accurate and consistent transport.
}
\label{fig:compare_meta_pfm}
\end{figure}
\vspace{-0.6em}

%\kaly{I like the idea of combining these two figures! But figure (a) is too small, because the legends and titles in the figure are unreadable! We need some smart engineering of the figures } \francisco{Agreed. Maybe if I put the bars horizontally we can have extra space for the figure on the left.} \kaly{Actually I will rerun and replot this! }

\begin{figure*}[!h]        
    \centering    
    \centerline{\includegraphics[width=1\columnwidth]{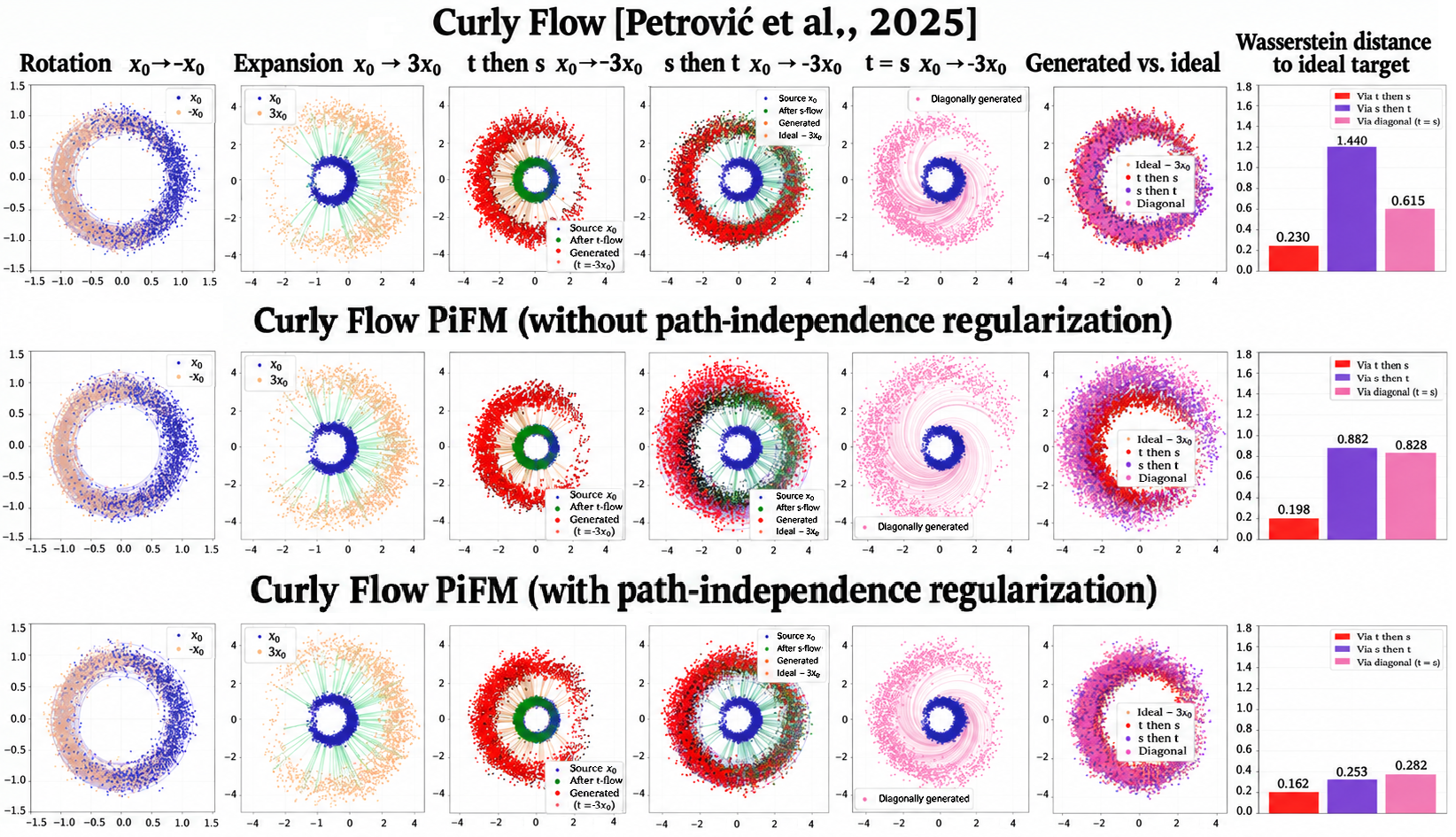}}
    \caption{\small Illustration of PiFM in the Curly Flow Matching setting~\citep{petrovic2025curly} together with a quantitative evaluation of path dependence. Curly-FM and unregularized PiFM produce path-dependent dynamics and inconsistent final distributions across integration strategies. In contrast, PiFM with path-independence regularization yields consistent transport and the lowest value in Wasserstein-2 distance to the target distribution.
}
    \label{fig:PFM_CurlyFlow_Results}
\end{figure*}

{\bf Comparison.} Figure~\ref{fig:meta_compare1} shows that MFM has the correct flow (in blue) and target samples in the training set (top and bottom row), but misses the target on the test set (middle row). On the contrary, Figure~\ref{fig:meta_compare2} shows that PiFM predicts the correct flow and target from the unseen source distribution (middle row) by inferring the source distributions as samples on a ``base" flow over the vertical direction. Moreover, we systematically evaluate the 2-Wasserstein distance between the predicted and expected target distributions over multiple runs of varying seeds. For the same number of batches (= 512) and training epochs (= 3000) for both MFM and PiFM over 10 random seeds, Figure~\ref{fig:seeds_compare} shows that the mean and standard deviation of 2-Wasserstein distance are much smaller for PiFM than for MFM, indicating the PiFM is more robust and accurate in this task. For same number of training epochs (= 2000) with varying batch sizes, PiFM consistently shows better performance (Figure~\ref{fig:batchsize_compare}).
\vspace{-0.5em}
\subsubsection{PiFM for Inferring Desirable Path-independent Distributions}\label{sec:Inference}
\vspace{-0.4em}

{\bf Toy dataset.}
We evaluate PiFM in the Curly Flow Matching setting \citep{petrovic2025curly}, where the
goal is to compose a rotation flow ($x_0 \to -x_0$) and an expansion flow ($x_0 \to 3x_0$)
to produce a joint transformation $x_0 \to -3x_0$. A distributionally path-independent model should yield
the same target distribution regardless of the integration strategy used at inference time.

{\bf Comparison.} As shown in Figure~\ref{fig:PFM_CurlyFlow_Results}, the original Curly-FM model, which trains the two vector fields separately, fails to achieve commutativity: sequential and diagonal integration paths produce inconsistent results, revealing path-dependent dynamics. PiFM without regularization also fails to recover commutativity despite joint training. In contrast, PiFM with path-independence regularization produces consistent outputs across all paths, with generated distributions closely matching the target distribution. This is further confirmed quantitatively in Figure~\ref{fig:PFM_CurlyFlow_Results}, where Curly-FM and unregularized PiFM exhibit large variations in Wasserstein-2 distance across integration strategies. PiFM with path-independence regularization achieves the lowest values, demonstrating that enforcing commutativity during training is key to reliable path-independent generation.

\vspace{-0.5em}

\subsection{High-Dimensional Setting: PiFM for Unsupervised Image-to-Image Translation}
\label{sec:Training_PFM_ImageTranslation}

\vspace{-0.4em}

We train and evaluate on the CelebA~\citep{CelebA} dataset using the \texttt{Smiling} and \texttt{Black\_Hair} attributes. We employ a shared U-Net backbone with two per-pixel MLP heads implemented as $1\times1$ convolutions with $\mathrm{SiLU}$ activations, each predicting the vector field associated with one attribute flow. A full description of the preprocessing and architecture is provided in Appendix~B.

\begin{wrapfigure}{r}{0.65\textwidth}      
    \centering    \centerline{\includegraphics[width=0.65\columnwidth]{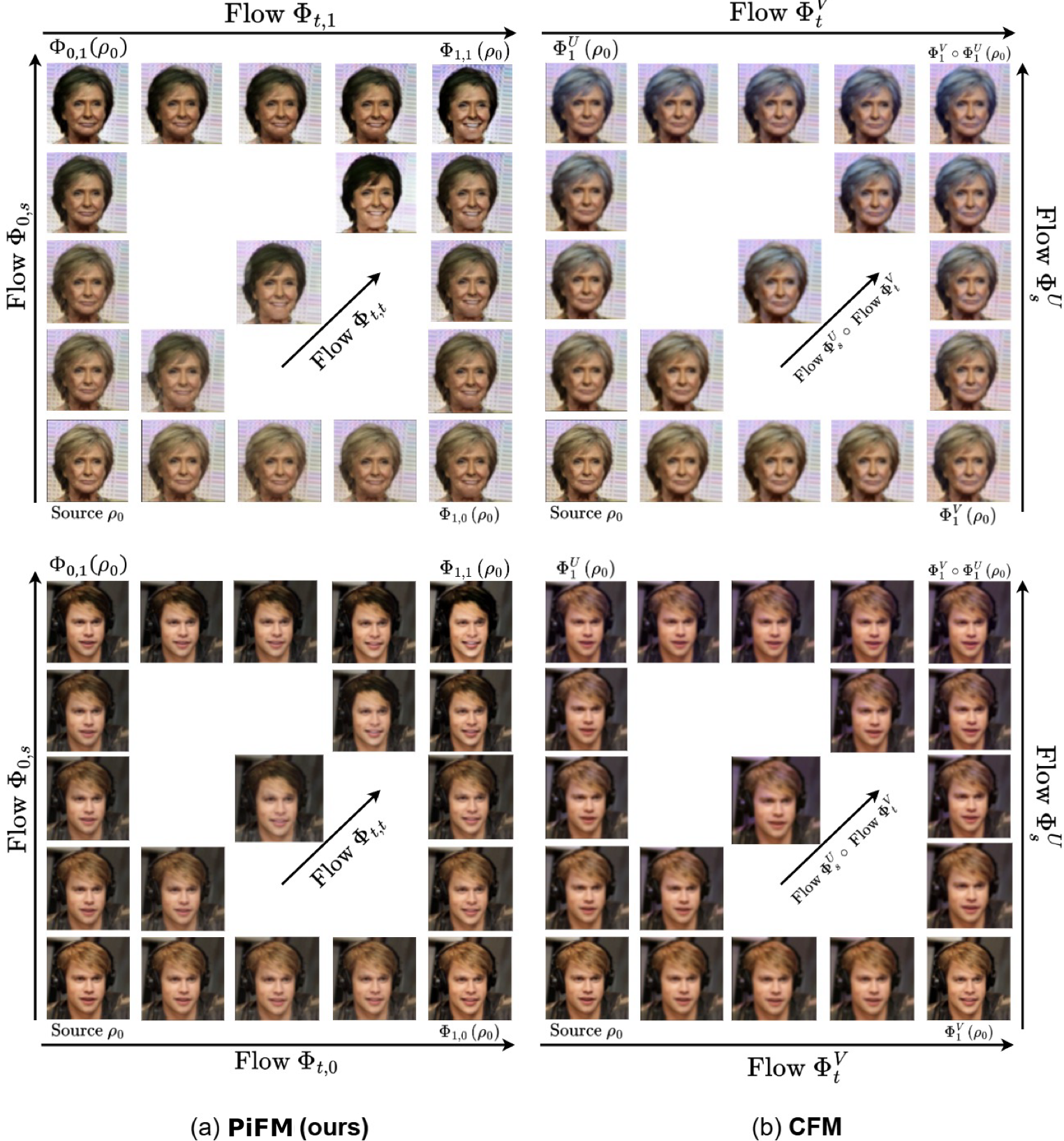}}
    \caption{\small We present an image-to-image translation comparison of PiFM on the CelebA~\citep{CelebA} dataset under three different flow integration strategies. PiFM successfully infers the missing attribute, whereas CFM fails to produce the desired transformation.}\label{fig:PFM_CelebA_JointAttr_Smile_BlackHair}
\end{wrapfigure}

Figure~\ref{fig:PFM_CelebA_JointAttr_Smile_BlackHair} shows qualitative results for female (panel (a) top and (b) top) and male (panel (a) bottom and (b) bottom) examples across the three different flow integration strategies: $t\to s$, $s\to t$ and diagonal. In the PiFM results, the A-field (\texttt{Smiling}) produces localized mouth/cheek edits consistent with smiling, while the B-field (\texttt{Black\_Hair}) darkens and re-textures the hair region. Note that the identity (global face shape, eyes) is largely preserved. The diagonal integration often produces smoother simultaneous edits to both attributes. In contrast, CFM often fails to modify the source image along either one or both flow axes (e.g., CFM cannot generate smiling faces in Figure~\ref{fig:PFM_CelebA_JointAttr_Smile_BlackHair}). The consistent inference of attributes in PiFM across samples and integration paths reiterates that PiFM has strong empirical path-independence properties at the distribution level. Likewise, individually each sample shows that across different integration paths approximately the same picture is obtained at the end, which indicates at least approximately pointwise path-independence is attained. Additional comparisons using FID are presented in Appendix~\ref{sec:appendix_fid}. We find FID poorly suited for evaluating commutativity as in Figure~\ref{fig:PFM_CelebA_JointAttr_Smile_BlackHair}, since baseline methods may generate images close in pixel space without achieving the desired attribute transformations.

\vspace{-0.5em}
\subsection{High-Dimensional Multi-Step Setting: PiFM for Single-Cell Reprogramming}
\vspace{-0.4em}
We further evaluate PiFM on a real-world single-cell RNA-seq reprogramming dataset~\citep{schiebinger2019optimal}. Data source and preprocessing details are provided in Appendix~\ref{sec:single-cell-preprocessing}. Cells are represented in 50-dimension PCA space, and the two biological axes are experimental day and pluripotency score. This setting provides a high-dimensional, multi-step biological test case: experimental day captures chronological reprogramming, while pluripotency score reflects acquisition of the target iPSC-like fate. Thus, the task evaluates whether PiFM can jointly model temporal dynamics and fate acquisition, rather than treating reprogramming as a single ordered trajectory.

As shown in Figure~\ref{fig:single-cell-main}, PiFM-inferred trajectories evolve coherently along both biological axes, capturing temporal reprogramming at fixed pluripotency and fate acquisition at fixed experimental day. Moreover, when transporting day 0 low-pluripotency cells to day 6 high-pluripotency targets, different integration orders yield overlapping trajectories and generated distributions that align with the real target population, supporting distributional path-independent generation. Biologically, this suggests that PiFM can model reprogramming as a multi-step and heterogeneous process, where cells observed at the same experimental day may occupy different intermediate states and only a subset progresses toward the pluripotent iPSC-like fate. We evaluate target prediction quality and commutativity in Appendix~\ref{sec:single-cell-evaluation}.

\vspace{-1em}
\begin{wrapfigure}{r}{0.65\textwidth}             
    \centering    \centerline{\includegraphics[width=0.65\columnwidth]{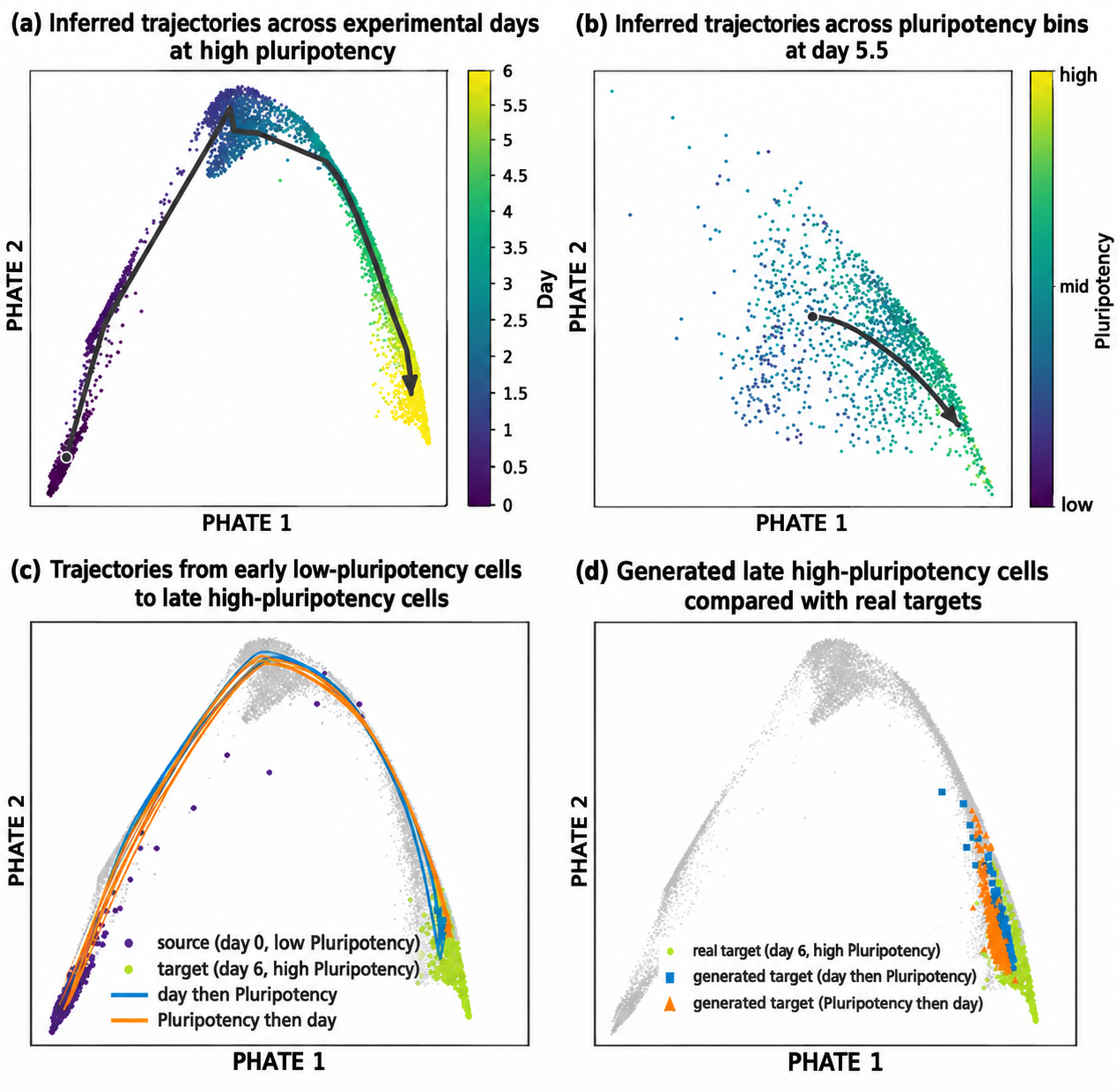}}
    \caption{\small PiFM-inferred single-cell reprogramming dynamics across experimental day and pluripotency progression, visualized by PHATE~\cite{moon2019visualizing}. (a) and (b) show trajectories along each biological axis, varying experimental day at fixed high pluripotency and varying pluripotency bins at fixed experimental day. (c) and (d) show path-ordered generation from day 0 low-pluripotency cells to day 6 high-pluripotency targets, where the two integration orders produce overlapping trajectories and generated distributions that match the real target population.}
    \label{fig:single-cell-main}
\end{wrapfigure}

\section{Discussion}\label{sec:Discussion}

\vspace{-0.8em}

%\kaly{I think we can use paragraphs like \textbackslash paragraph\{\}  rather than the subsections, that will save a lot of space} \sina{I agree}

\paragraph{Conclusion} We introduced Path-Independent Flow Matching (PiFM), a multi-parameter Flow Matching that enforces consistent path-independent transport of distributions. We also make a connection between PiFM and Wasserstein barycenters. We showed that empirically, PiFM outperforms existing approaches on synthetic and image translation tasks, particularly in order-invariant and out-of-distribution settings. We further demonstrated applications to Curly Flow Matching and single-cell RNA-seq trajectory inference.

\vspace{-0.8em}

\paragraph{Limitations}  While our theory provides conditions for path-independent transport at the distribution level, it does not guarantee path independence for individual samples. Moreover, the connection to the Wasserstein barycenter is only studied in the $n=2$ setting; further investigation is needed to establish the result in this work for the more general case.

\vspace{-0.8em}
\paragraph{Future Work} Promising future directions include extending PiFM to more general geometries, as in \cite{ChenRiemannianFM}, and further applying it to single-cell perturbation data and trajectory inference tasks. Additionally, we will focus on validating PiFM as a Wasserstein barycenter generative framework beyond the setting discussed here.

\vspace{-0.8em}
\paragraph{Aknowledgements} We thank Professor Paul François for his mentorship and support of Philippe in his capacity as a co-supervisor during the course of this project. This research was supported in part by Canada CIFAR AI Chair [G.W.], NSERC Discovery Grant 03267 [G.W.], NSF grant DMS-2327211 [G.W.], FRQNT–NSERC grant 2023-NOVA-329125 [G.W.; E.B.], Open Philanthropy [E.B.; Y.Z.] and NSERC Discovery grant 05439 [S.S.]. The authors are solely responsible for the content of this work, and the views expressed in it do not necessarily reflect those of the funding agencies. This research was enabled in part by compute resources provided by Mila, Université de Montréal, Compute Canada, and Calcul Québec. We sincerely thank the anonymous reviewers and Area Chairs for their constructive feedback and efforts throughout the review process, which helped us improve the paper.

\vspace{-0.8em}
\paragraph{Impact statement}
\label{sec:Impact_Statement} The authors acknowledge that this work may have broader societal implications; however, no implications are identified as requiring further specific discussion in the context of this manuscript.

\vspace{-0.8em}
\paragraph{Code availability}
Our source code is publicly available at \href{https://github.com/YanleiZhang01/PiFM}{github.com/YanleiZhang01/PiFM}.

%\vspace{-1em}
%\paragraph{Supplementary material} Our code is available at \url{https://anonymous.4open.science/r/Pi_flow_matching-C83F/}.

%\vspace{-0.8em}
%\section{Aknowledgements}
%\label{sec:Impact_Statement}

%\newpage

%\vspace{-0.8em}

%\section{Impact Statement}
%\kaly{please check author instructions on website if this is needed!}

%\vspace{-0.8em}

%The authors acknowledge that this work may have societal implications, although none are considered sufficiently significant to warrant specific discussion here.

% \begin{ack}
% Use unnumbered first level headings for the acknowledgments. All acknowledgments

% Do {\bf not} include this section in the anonymized submission, only in the final paper. You can use the \texttt{ack} environment provided in the style file to automatically hide this section in the anonymized submission.
% \end{ack}

%\section*{References}

% References follow the acknowledgments in the camera-ready paper. Use unnumbered first-level heading for
% the references. Any choice of citation style is acceptable as long as you are
% consistent. It is permissible to reduce the font size to \verb+small+ (9 point)
% when listing the references.
% Note that the Reference section does not count towards the page limit.

\newpage

\bibliographystyle{neurips_2026}
\bibliography{main}
\medskip

%%%%%%%%%%%%%%%%%%%%%%%%%%%%%%%%%%%%%%%%%%%%%%%%%%%%%%%%%%%%

\appendix

% Technical appendices with additional results, figures, graphs, and proofs may be submitted with the paper submission before the full submission deadline (see above). You can upload a ZIP file for videos or code, but do not upload a separate PDF file for the appendix. There is no page limit for the technical appendices. 

% Note: Think of the appendix as ``optional reading'' for reviewers. The paper must be able to stand alone without the appendix; for example, adding critical experiments that support the main claims to an appendix is inappropriate. 
\newpage
\section{Theoretical Details and Proofs}
\label{sec:theory}
\subsection{The general multi-parameter case}\label{sec:generalization}

PiFM generalizes to the case of multiple parameters $t_1, \ldots, t_n$. In this setting, we consider a source data distribution \(q(a)\) and target data distributions \(q(b_1), \ldots, q(b_n)\). We condition on the joint distribution \(q(z)\), where \(z = (a, b_1, \ldots, b_n)\), with a marginalization,
\begin{equation}\label{eq:gen_marginal}
    p_{t_1,\ldots,t_n}(x)=\int p_{t_1,\ldots,t_n}(x\mid z)q(z)\,dz.
\end{equation}
The continuously differentiable vector fields $u_{t_1,\ldots,t_n}^{(i)}$ are associated with the ODEs
\begin{equation}\label{eq:gen_ODE}
    \frac{\partial \phi^{(i)}(x)_{t_1,\ldots,t_n}}{\partial t_i}
    = u_{t_1,\ldots,t_n}^{(i)}\big(\phi_{t_1,\ldots,t_n}(x)\big)
\end{equation}
Each vector field \(u_{t_1,\ldots,t_n}^{(i)}\) governs the evolution from the source distribution \(q(a)\) to the target distribution \(q(b_i)\) and has an associated flow $\Phi_{t_1,...,t_n}:[0,1]^{n}\times \mathbb{R}^{d}\to \mathbb{R}^{d},\Phi_{t_1,...,t_n}(x)=\phi_{t_1,...,t_n}(x)$. We define
\begin{equation}\label{eq:gen_VF}
    u_{t_1,\ldots,t_n}^{(i)}(x)
    = \mathbb{E}_{q(z)}\left[
        \frac{u_{t_1,\ldots,t_n}^{(i)}(x | z)\, p_{t_1,\ldots,t_n}(x | z)}
        {p_{t_1,\ldots,t_n}(x)}
    \right]
\end{equation}
for all $i = 1, \ldots, n$. 

Here, $u_{t_1,...,t_n}^{(i)}(x|z)$ are the conditional vector fields. We consider the $n$ continuity equations
\begin{equation}\label{eq:gen_continuity}
    \frac{\partial p_{t_1,\ldots,t_n}(x)}{\partial t_i}
    + \nabla \cdot \big(p_{t_1,\ldots,t_n}(x)\, u_{t_1,\ldots,t_n}^{(i)}(x)\big)
    = 0,
\end{equation}
which hold for all \(t_j \in [0,1]\) with \(j \neq i\). Additionally, the path-independence property \eqref{eq:path_independence} is replaced by a family of equalities corresponding to integrating along each parameter direction sequentially, resulting in \(n!\) equivalent expressions. Finally, \eqref{eq:integrability} is replaced by \(\binom{n}{2}\) conditions encoding pairwise consistency between vector fields. For all \(i,j = 1,\ldots,n\), \(i \neq j\), we obtain
\begin{equation}\label{eq:gen_Lie}
    \partial_{t_j}u_{t_1,...,t_n}^{(i)}(x)-\partial_{t_i}u_{t_1,...,t_n}^{(j)}(x)=[u_{t_1,...,t_n}^{(i)}(x),u_{t_1,...,t_n}^{(j)}(x)]
\end{equation}

The generalized version of the affine probability path and conditional vector fields is as follows. We take
\begin{equation*}
    u_{t_1,...,t_n}^{(i)}(x|z)=b_i-a \hspace{1em},\forall i=1,...,n
\end{equation*}
and
\begin{equation*}
    p_{t_1,...,t_n}(x|z)=\mathcal{N}(x|\mu_{t_1,...,t_n}(z),\sigma^{2})
\end{equation*}
with constant $\sigma>0$ and $\mu_{t_1,...,t_n}(z)=a+\sum_{i=1}^{n}(b_i-a)t_i$. As discussed in the main part of the paper, the choice of conditional vector fields and probability paths is a user defined choice as long as there is consistent and path-independent distribution transport.

%For ease of notation we set $T=(t_1,...,t_n)$ so that $p_{t_1,...,t_n}$ is more succinctly expressed as $p_{T}$ and likewise $u_{t_1,...,t_n}^{(i)}$ as $u_{T}^{(i)}$.

%\begin{align*}
%&\frac{\partial (u^{(i)})^{p}}{\partial t_j}
%+ \sum_{k=1}^{d} (u^{(j)})^{k} \frac{\partial (u^{(i)})^{p}}{\partial x_k}
%= \frac{\partial (u^{(j)})^{p}}{\partial t_i} + \sum_{k=1}^{d} (u^{(i)})^{k} \frac{\partial (u^{(j)})^{p}}{\partial x_k}.
%\end{align*}

\subsection{Proofs of theorems}

\begin{lemma}[General form Lemma \ref{lem:path}]
    Consider $u_{t_1,...,t_n}^{(i)}(x)$ vector fields. If there exists a unique generated probability density path $p_{t_1,...,t_n}(x)$ satisfying \eqref{eq:gen_continuity}, then the associated flows $\Phi_{t_1,...,t_n}^{(i)}(x)$ satisfy the path-independence condition.
\end{lemma}

\begin{proof}
    Suppose that the vector fields $u_{T}^{(i)}(x)$ generate $p_{T}(x)$,
\begin{equation*}
    \frac{\partial p_{T}(x)}{\partial t_i} + \nabla \cdot \big(p_{T}(x)\, u_{T}^{(i)}(x)\big) = 0, 
    \qquad \forall i = 1, \dots, n.
\end{equation*}

There are $n!$ possible orderings in which one can successively transport along each coordinate, starting from $(0,\dots,0)$ and reaching $T = (t_1,\dots,t_n)$. Without loss of generality, it suffices to consider the ordering $t_1 \to t_2 \to \cdots \to t_n$, as all other cases follow analogously.

We proceed iteratively. First, we transport using the flow $\Phi_{t_1,0,\dots,0}^{(1)}$ which is associated to the vector field $u_{T}^{(1)}$, obtaining $(\Phi_{t_1,0,\dots,0}^{(1)})_{\#} p_{0,\dots,0}$.
Next, we transport using $\Phi_{t_1,t_2,0,\dots,0}^{(2)}$, yielding $(\Phi_{t_1,t_2,0,\dots,0}^{(2)})_{\#} \circ (\Phi_{t_1,0,\dots,0}^{(1)})_{\#} \, p_{0,\dots,0}$. Continuing in this manner, we arrive at the expression
\begin{equation*}
    (\Phi_{t_1,\dots,t_n}^{(n)})_{\#} \circ \cdots \circ (\Phi_{t_1,t_2,0,\dots,0}^{(2)})_{\#} \circ (\Phi_{t_1,0,\dots,0}^{(1)})_{\#} \, p_{0,\dots,0}.
\end{equation*}
Since the vector fields $u_{T}^{(i)}$ jointly generate $p_{T}(x)$, it follows that at $T = (t_1,\dots,t_n)$ we have
\begin{equation*}
(\Phi_{t_1,\dots,t_n}^{(n)})_{\#} \circ \cdots \circ (\Phi_{t_1,t_2,0,\dots,0}^{(2)})_{\#} \circ (\Phi_{t_1,0,\dots,0}^{(1)})_{\#} \, p_{0,\dots,0}
= p_{t_1,\dots,t_n}.
\end{equation*}
    Repeating the same argument for the remaining $n! - 1$ orderings, and noting that each construction yields $p_{t_1,\dots,t_n}$, we conclude that all $n!$ resulting expressions coincide.
\end{proof}

\begin{proposition}[General form of Proposition~\ref{prop:alternative_path}]
    Suppose that the vector fields $u_{t_1,..,t_n}^{(i)}(x)$ satisfy the integrability condition \eqref{eq:gen_Lie}
  Then there exists a unique flow $\Phi_{t_1,...,t_n}$ which satisfies the path-independence property.
\end{proposition}

\begin{proof}
    Suppose that the vector fields $u_{T}^{(i)}$ satisfy the integrability condition
\begin{equation*}
    \partial_{t_i}u_{t_1,...,t_n}^{(i)}-\partial_{t_j}u_{t_1,...,t_n}^{(j)}=[u_{t_1,...,t_n}^{(i)},u_{t_1,...,t_n}^{(j)}]
\end{equation*}
    By Proposition~19.29 in \citep{lee2003introduction}, there exists a unique solution $\phi_{t_1,\dots,t_n}$ to the system
\begin{equation*}
    \frac{\partial \phi_{t_1,\dots,t_n}(x)}{\partial t_i}
    = u_{t_1,\dots,t_n}^{(i)}\big(\phi_{t_1,\dots,t_n}(x)\big),
    \qquad \forall i = 1, \dots, n,
\end{equation*}
In this case, the map $\Phi : [0,1]^n \times \mathbb{R}^d \to \mathbb{R}^d$, defined by
\[
\Phi_{t_1,\dots,t_n}(x) = \phi_{t_1,\dots,t_n}(x),
\]
is the unique flow associated with the family of vector fields $u_{T}^{(i)}$. It follows that transporting along the flows $\Phi^{(i)}$, which in this case is $\Phi$, in any order yields the same result.
\end{proof}

\begin{theorem}[General form of Theorem \ref{thm:gen}]
    Let the conditional vector fields $u_{t_1,...,t_n}^{(i)}(x|z)$ generate the unique conditional probability density path  $p_{t_1,...,t_n}(x|z)$ satisfying the analogous equations \eqref{eq:gen_continuity}. Then, the vector fields $u_{t_1,...,t_n}^{(i)}(x)$ in \eqref{eq:gen_VF} generate the unique probability density path $p_{t_1,...,t_n}(x)$ satisfying \eqref{eq:gen_continuity}.
\end{theorem}
\label{sec:Proofs_of_theorems}
\begin{proof}
    Suppose that the conditional vector fields $u_{t_1,...,t_n}^{(i)}(x|z)$ generate $p_{t_1,...,t_n}(x|z)$. That is,
    \begin{equation*}
        \frac{\partial p_{t_1,...,t_n}(x|z)}{\partial t_i}+\nabla \cdot (p_{t_1,...,t_n}(x|z)u_{t_1,...,t_n}^{(i)}(x|z))=0 \hspace{1em},\forall i=1,...,n.
    \end{equation*}
    Then,
    \begin{align*}
        -\nabla \cdot (p_{t_1,...,t_n}(x)u_{t_1,...,t_n}^{(i)}(x))
        &=-\nabla \cdot \left(p_{t_1,...,t_n}(x)\int u_{t_1,...,t_n}^{(i)}(x|z)\frac{p_{t_1,...,t_n}(x|z)q(z)}{p_{t_1,...,t_n}(x)}dz\right)\\
        &=-\nabla \cdot \left(\int u_{t_1,...,t_n}^{(i)}(x|z)p_{t_1,...,t_n}(x|z)q(z)dz\right)\\
        &=-\int \nabla \cdot \left(u_{t_1,...,t_n}^{(i)}(x|z)p_{t_1,...,t_n}(x|z)q(z)\right)dz\\
        &=-\int \nabla \cdot \left(u_{t_1,...,t_n}^{(i)}(x|z)p_{t_1,...,t_n}(x|z)\right)q(z)dz\\
        &=\int \frac{\partial p_{t_1,...,t_n}(x|z)}{\partial t_i}q(z)dz\\
        &=\frac{\partial }{\partial t_i}\int p_{t_1,...,t_n}(x|z)q(z)dz\\
        &=\frac{\partial p_{t_1,...,t_n}(x)}{\partial t_i}.
    \end{align*}
    Therefore, the vector fields $u_{t_1,...,t_n}^{(i)}$ generate $p_{t_1,...,t_n}(x)$.
\end{proof}

\begin{corollary}[General form of Corollary \ref{thm:path_independence}]
    Suppose that the vector fields $u_{t_1,...,t_n}(x|z)$ generate the unique conditional probability density path  $p_{t_1,...,t_n}(x|z)$ satisfying the analogous equations \eqref{eq:gen_continuity}. Then, the associated flows $\Phi_{t_1,...,t_n}^{(i)}$ of the vector fields $u_{t_1,...,t_n}^{(i)}(x)$ given in \eqref{eq:gen_VF} satisfy the path-independence property.
\end{corollary}

\begin{proof}
    Suppose that the conditional vector fields $u_{t_1,...,t_n}^{(i)}(x \mid z)$ generate $p_{t_1,...,t_n}(x \mid z)$. By Theorem~\ref{thm:gen}, it follows that the vector fields $u_{t_1,...,t_n}^{(i)}(x)$ generate $p_{t_1,...,t_n}(x)$. Finally, by Lemma~\ref{lem:path}, transporting along the flows $\Phi_{t_1,...,t_n}^{(i)}$ in any order yields the same result.
\end{proof}

\subsection{Discussion about the relation between PiFM and Wasserstein Barycenter}
\label{sec:appendix_wb_relation}
We hereafter discuss Theorem \ref{theorem:WB_exact_matching} and the connection between PiFM and Wasserstein Barycenter.

\begin{lemma}[Lemma 3.1 of \cite{boissard2013distributionstemplateestimatewasserstein}]\label{Lemma:3.1-App}

Let  $\Pi(\mu_1,\ldots,\mu_K )$ be the set of probability measures on $(\mathbb{R}^d)^K$ with marginals $\mu_1,\ldots,\mu_K$, respectively, and $$T (x_1,\ldots,x_K )= \underset{1\leq j\leq K}{\sum} \lambda_jx_j$$ with  $\lambda_j \geq 0$ such that $\underset{1\leq j\leq K}{\sum} \lambda_j = 1$. A probability measure $\nu$ is a barycenter of  $\mu_1,\ldots, \mu_K $ with weights $(\lambda_j)_{1\leq j\leq K}$ if and only if $\nu = T\#\gamma$ where $\gamma \in \Pi(\mu_1, \ldots, \mu_K)$ minimizes the multi marginal Optimal Transport:  
\begin{align} \label{eq:J-gam}
J(\gamma, \lambda)= \int \sum_{j=1}^{K}\lambda_j \left\| T(x_1, \ldots, x_K) - x_j\right \|^2d\gamma(x_1, \ldots, x_K) \end{align}
\end{lemma}

%\phil{TODO: change numbering and calls to (1), (2)}
By choosing $T$ as $T(a,b,c)  = a + t(b-a) + s(c-a) = (1-t-s) a + tb + sc$ and  $ \lambda = (1-t-s, t,s)$ we have: \[
J(\gamma,\lambda)
= \int \Big[
(1-t-s)\|T(a,b,c)-a\|^2
+ t\|T(a,b,c)-b\|^2
+ s\|T(a,b,c)-c\|^2
\Big] \, d\rho_0(a)\rho_1(b)\rho_2(c)
\]
With Optimal Coupling such that $b = T_1(a)$ and $c=T_2(a)$, $T$ only depends on $a$ and we have $T(a,b,c) = z_{t,s}(a)$ and the associated multi marginal transport is defined as
\begin{align}\label{eq:J-PiFM2}
J^{\text{PiFM}}(\rho_0, \lambda)
:= \int \Big[
(1-t-s)\| z_{t,s}(a)-a \|^2
+ t\| z_{t,s}(a)-T_1(a) \|^2
+ s\| z_{t,s}(a)-T_2(a) \|^2
\Big]\, d\rho_0(a)
\end{align}
Without specific assumptions, $J^{\text{PiFM}}(\rho_0, \lambda)$ is not guaranteed to be a minimizer of \eqref{eq:J-gam}. In the following, we give sufficient conditions referred to as the admissible family of deformations under which $J^{\text{PiFM}}(\rho_0, \lambda)$  minimizes \eqref{eq:J-gam}.

\begin{theorem}[Theorem 4.1 of \cite{boissard2013distributionstemplateestimatewasserstein}]\label{Thm4.1-App}
Suppose that $(T_i)_{1\leq i \leq K}$ is an admissible family of deformations on a domain $\Omega \subset \mathbb{R}^n$, and let $\mu \in \mathcal{P}_2(\Omega)$ be absolutely continuous with respect to the Lebesgue measure. Let $\mu_j = (T_j)_{\#}\mu$,  $\mu = (\mu_j)_{1 \leq j \leq K}$ and $\lambda = (\lambda_j)_{1 \leq j \leq K}$. Then,
\begin{align}\label{eq:WB-3}
    WB(\mu, \lambda) = \left(\sum_{j=1}^{K} \lambda_j T_j\right)_{\#} \mu
\end{align}
\end{theorem}

The notion of an admissible family of deformations is defined in \cite{boissard2013distributionstemplateestimatewasserstein} (Definition 4.2) as follows:
\begin{enumerate}
    \item there exists $j_0$ such that $T_{j_0} = \text{Id}$
    \item for each $1 \leq j \leq K$:  $T_j$  is one-to-one and onto,
    \item the composition $T_i \circ T_j^{-1}$ of two maps in this class remains an optimal one for all $1 \leq i, j \leq  K$
\end{enumerate} 

We now prove Theorem  \ref{theorem:WB_exact_matching}.

\begin{proof}[Proof of Theorem \ref{theorem:WB_exact_matching}]
Let's first recall that with optimal coupling such that $b = T_1(a)$ and $c=T_2(a)$ and in the limit of $\sigma \rightarrow 0$, we have $T(a,b,c) = z_{t,s}(a)  = a +(1-t)T_1(a)  + (1-s) T_2(b)$.  Therefore, by construction: 
\begin{align}
\hat{\rho}_{t,s} :=& \Gamma_{t,s}(\rho_0) \\    
:=& \left((1-t-s)\text{Id} + tT_1 +sT_2\right))_{\#}\rho_0  \label{eq:thing}  
\end{align}
where we recall that $\Gamma_{t,s}(\rho_0)=(\Phi_{t,s})_{\#}\circ (\Psi_{0,s})_{\#}\rho_0=(\Psi_{t,s})_{\#}\circ (\Phi_{t,0})_{\#}\rho_0$. Equation \eqref{eq:thing} can be rewritten as
$$ 
\hat{\rho}_{t,s} = \left(\sum_{j=1}^{K} \lambda_j T_j\right)_{\#} \rho_0 \\ 
$$ 
Assuming that $T_1$ and $T_2$ belong to the admissible family of deformation, using Theorem \ref{Thm4.1-App} we get that $\hat{\rho}_{t,s} $  is exactly the Wasserstein Barycenter, i.e PiFM generates the Wasserstein barycenter of $\rho_0, \rho_1, \rho_2$ with $\lambda = (1-t-s,t,s)$, this concludes the proof. 

Additionally, using \eqref{eq:J-PiFM2} and applying Lemma \ref{Lemma:3.1-App}, we get:
\begin{align*}
J^{\text{PiFM}}(\rho_0, \lambda) =& \int \left[(1-t-s) \left|| z_{t,s}(a) - a|\right |^2 + t\left|| z_{t,s}(a) - T_1(a)|\right |^2 + s \left|| z_{t,s}(a) - T_2(a)|\right |^2 \right]d\rho_0(a)  \\
=& \underset{\gamma}{argmin}\ J(\gamma,\lambda)
\end{align*}

\end{proof}

\subsection{Additional results relating PiFM to the Wasserstein Barycenter}

We note that not only the model predicts the Wasserstein barycenter within the simplex defined by $\rho_0$, $\rho_1$ and $\rho_2$, but we also empirically observe that for toy distributions outside of this simplex, $\rho^{\text{PiFM}}_{t, s}$ coincides with the horizontal barycenter $\rho_s^U(t) $ (resp. vertical barycenter $\rho_t^V(s) $) after applying $\Gamma_{0,s}$ (resp. $\Gamma_{t,0}$) to the initial and target distributions as shown in Supplementary Figure~\ref{fig:PV_vs_PU_vs_WB} and in Supplementary Figure~\ref{fig:WB_U_V_explained}. Below, we discuss the Gaussian case with proof related to horizontal and vertical Wasserstein barycenter. 

% \iffalse
% \sina{Conditions under which this statement holds for general complex distributions remain an open problem. Remove this: For Gaussian distributions and empirically in the toy settings considered here, the observed coincidence between the global, PiFM, horizontal, and vertical barycenters using $\lambda = (1-t-s, t, s) $ can be summarized as
% $WB_{\lambda}(\rho_0, \rho_1, \rho_2) = \rho_{t, s}^{\text{PiFM}} = \rho_s^U(t) = \rho_t^V(s)$
% with
% \begin{align*}
% \rho_s^U(t)
% &= \underset{\rho}{\arg\min}\;
% (1-t)\,\mathcal{W}_2^2\!\big(\phi_s^V(\rho_0), \rho\big)\nonumber
% + t\,\mathcal{W}_2^2\!\big(\phi_s^V(\rho_1), \rho\big), \\
% \rho_t^V(s)
% &= \underset{\rho}{\arg\min}\;
% (1-s)\,\mathcal{W}_2^2\!\big(\phi_t^U(\rho_0), \rho\big)\nonumber
%  +s\,\mathcal{W}_2^2\!\big(\phi_t^U(\rho_2), \rho\big).
% \end{align*}}
% \sina{shouldn't $\phi_t^U(\rho_0)$ be $(\Phi_{t,0})_{\#}\rho_0$?}

% \documentclass{article}
% \usepackage{amsmath, amssymb}
% \usepackage{amsthm}
% \usepackage{tikz}
%\newtheorem{definition}{Definition}

% \newtheorem{theorem}{Theorem}
% \newtheorem{assumption}{Assumption}
% \newtheorem{prop}{Proposition}
% \newtheorem{cor}{Corollary}
% \newtheorem{remark}{Remark}
%\begin{document}

\begin{theorem}\label{thm:gaussian-barycenter}
Let $m_0,m_1,m_2\in\mathbb R^d$. Suppose that there exists a symmetric positive semidefinite matrix \(\Sigma\in\mathbb R^{d\times d}\) such that
\[
\rho(0)=\mathcal N(m_0,\Sigma),\qquad
\rho(1)=\mathcal N(m_1,\Sigma),\qquad
\rho(2)=\mathcal N(m_2,\Sigma).
\]
Define the conditional vector fields
\[
\hat{u}_{t,s}(x):=u_{t,s}(x|z):=m_1-m_0,
\qquad
\hat{v}_{t,s}(x):= v_{t,s}(x|z):=m_2-m_0,
\]
and, for \((t,s)\in\Delta^2:=\{(t,s)\in[0,1]^2:\ t+s\le 1\}\), define
\[
\hat Q_{t,s}
:=
\mathcal N\!\Bigl((1-t-s)m_0+t\,m_1+s\,m_2,\Sigma\Bigr).
\]
Let \(\mathcal Q_{t,s}\) denote the probability measure with density \(\hat Q_{t,s}\).

Then:

\begin{enumerate}
\item [(i)] The family \(\hat Q_{t,s}\) satisfies the two continuity equations
\[
\partial_t \hat Q_{t,s}+\nabla\cdot(\hat Q_{t,s}\,\hat u)=0,
\qquad
\partial_s \hat Q_{t,s}+\nabla\cdot(\hat Q_{t,s}\,\hat v)=0,
\]
with boundary conditions
\[
\hat Q_{0,0}=\rho(0),\qquad
\hat Q_{1,0}=\rho(1),\qquad
\hat Q_{0,1}=\rho(2).
\]

\item [(ii)] For every \((t,s)\in\Delta^2\), \(\mathcal Q_{t,s}\) is the Wasserstein barycenter of
\(\rho(0),\rho(1),\rho(2)\) with weights \((1-t-s,t,s)\), i.e.,
\[
\mathcal Q_{t,s}
=
\operatorname*{arg\,min}_{\nu\in\mathcal P_2(\mathbb R^d)}
\Bigl\{
(1-t-s)\mathcal{W}_2^2(\nu,\rho(0))
+t\,\mathcal{W}_2^2(\nu,\rho(1))
+s\,\mathcal{W}_2^2(\nu,\rho(2))
\Bigr\}.
\]

\item [(iii)] The horizontal and vertical slice barycenter identities also hold:

\begin{itemize}
    \item [(a)] For each fixed \(s\in[0,1]\), define
\[
A_s:=(\Psi_s^{0})_\#\rho(0),\qquad
B_s:=(\Psi_s^{1})_\#\rho(1).
\]
Then
\[
A_s=\mathcal N(m_0+s(m_2-m_0),\Sigma),
\qquad
B_s=\mathcal N(m_1+s(m_2-m_0),\Sigma),
\]
and for every \(t\in[0,1]\),
\[
\mathcal Q_{t,s}
=
\operatorname*{arg\,min}_{\nu\in\mathcal P_2(\mathbb R^d)}
\Bigl\{
(1-t)\mathcal{W}_2^2(\nu,A_s)+t\,\mathcal{W}_2^2(\nu,B_s)
\Bigr\}.
\]

\item [(b)] For each fixed \(t\in[0,1]\), define
\[
C_t:=(\Phi_t^{0})_\#\rho(0),\qquad
D_t:=(\Phi_t^{1})_\#\rho(2).
\]
Then
\[
C_t=\mathcal N(m_0+t(m_1-m_0),\Sigma),
\qquad
D_t=\mathcal N(m_2+t(m_1-m_0),\Sigma),
\]
and for every \(s\in[0,1]\),
\[
\mathcal Q_{t,s}
=
\operatorname*{arg\,min}_{\nu\in\mathcal P_2(\mathbb R^d)}
\Bigl\{
(1-s)\mathcal{W}_2^2(\nu,C_t)+s\,\mathcal{W}_2^2(\nu,D_t)
\Bigr\}.
\]
\end{itemize}

\end{enumerate}
\end{theorem}

\begin{proof}[Proof of Theorem \ref{thm:gaussian-barycenter}]

Set
\[
\gamma:=\mathcal N(0,\Sigma),
\qquad
\tau_m(x):=x+m.
\]
Then
\[
\rho(i)=(\tau_{m_i})_\#\gamma,
\qquad i=0,1,2.
\]

\medskip
\noindent
\textbf{Step 1: Explicit form of the flow and continuity equations.}

Since \(\hat u\) and \(\hat v\) are constant vector fields, their flows are translations:
\[
\Phi_{t,s}(x)=x+t(m_1-m_0),
\qquad
\Psi_{t,s}(x)=x+s(m_2-m_0).
\]
Hence
\[
(\Phi_{t,s})_\#\rho(0)
=
\mathcal N(m_0+t(m_1-m_0),\Sigma),
\]
and
\[
(\Psi_{t,s})_\#\rho(0)
=
\mathcal N(m_0+s(m_2-m_0),\Sigma).
\]
Hence,
\[
\mathcal Q_{t,s}
=
\mathcal N\!\Bigl(m_0+t(m_1-m_0)+s(m_2-m_0),\Sigma\Bigr)
=
\hat Q_{t,s}.
\]
Since \(\hat Q_{t,s}\) is a translated Gaussian with constant covariance and
its mean evolves according to the constant velocities \(\hat u,\hat v\),
it satisfies
\[
\partial_t \hat Q_{t,s}+\nabla\cdot(\hat Q_{t,s}\,\hat u)=0,
\qquad
\partial_s \hat Q_{t,s}+\nabla\cdot(\hat Q_{t,s}\,\hat v)=0.
\]
The boundary conditions are immediate from the definition.

\medskip
\noindent
\textbf{Step 2: Global barycenter identity.}

Let
\[
\lambda_0:=1-t-s,\qquad \lambda_1:=t,\qquad \lambda_2:=s,
\qquad (\lambda_0,\lambda_1,\lambda_2)\in\Delta^2.
\]
Set
\[
\bar m_{t,s}:=\lambda_0 m_0+\lambda_1 m_1+\lambda_2 m_2.
\]
Then
\[
\mathcal Q_{t,s}
=
\mathcal N(\bar m_{t,s},\Sigma)
=
(\tau_{\bar m_{t,s}})_\#\gamma.
\]

We claim that \(\mathcal Q_{t,s}\) minimizes
\[
J(\nu, \lambda):=\sum_{i=0}^2 \lambda_i \mathcal{W}_2^2(\nu,\rho(i)).
\]
Take any \(\nu\in\mathcal P_2(\mathbb R^d)\), and let \(m_\nu\) denote its mean.
Write \(\tilde \nu:=(\tau_{-m_\nu})_\#\nu\), so that \(\tilde \nu\) has mean zero.

For each \(i\), since \(\rho(i)=(\tau_{m_i})_\#\gamma\), any coupling of \(\nu\) and \(\rho(i)\) may be written in terms of random variables \(Y\sim \nu\), \(X\sim \gamma\). Expanding the square gives (because the cross term vanishes): 
\[
\mathbb E|Y-(X+m_i)|^2
=
\mathbb E|(Y-m_\nu)-X|^2 + |m_\nu-m_i|^2,
\]
Taking the infimum over couplings yields
\[
\mathcal{W}_2^2(\nu,\rho(i))
\ge
\mathcal{W}_2^2(\tilde\nu,\gamma)+|m_\nu-m_i|^2.
\]
Therefore
\[
J(\nu, \lambda)
\ge
\mathcal{W}_2^2(\tilde\nu,\gamma)+\sum_{i=0}^2 \lambda_i |m_\nu-m_i|^2.
\]
Since \(\sum_i\lambda_i=1\), the quadratic function
\[
m\longmapsto \sum_{i=0}^2 \lambda_i |m-m_i|^2
\]
is uniquely minimized at
\[
m=\sum_{i=0}^2 \lambda_i m_i=\bar m_{t,s}.
\]
Also,
\[
\mathcal{W}_2^2(\tilde\nu,\gamma)\ge 0,
\]
with equality if and only if \(\tilde\nu=\gamma\), i.e.
\[
\nu=(\tau_{\bar m_{t,s}})_\#\gamma
=\mathcal Q_{t,s}.
\]
Thus \(\mathcal Q_{t,s}\) is the unique minimizer of \(F\).

This step could also be simply demonstrated using Theorem \ref{theorem:WB_exact_matching} as $\Phi_t^{s}$ and $\Psi_s^{t}$ belong to an admissible family of deformation.

\medskip
\noindent
\textbf{Step 3: Horizontal slice barycenter identity.}

Fix \(s\in[0,1]\). Since \(\Psi_{t,s}\) is translation by \(s(m_2-m_0)\),
\[
A_s=(\Psi_{0,s})_\#\rho(0)
=\mathcal N(m_0+s(m_2-m_0),\Sigma),
\]
and
\[
B_s=(\Psi_{1,s})_\#\rho(1)
=\mathcal N(m_1+s(m_2-m_0),\Sigma).
\]
Their weighted \((1-t,t)\)-barycenter is therefore the Gaussian with the same
covariance \(\Sigma\) and mean
\[
(1-t)\bigl(m_0+s(m_2-m_0)\bigr)
+t\bigl(m_1+s(m_2-m_0)\bigr)
=
m_0+t(m_1-m_0)+s(m_2-m_0),
\]
which is exactly the mean of \(\mathcal Q_{t,s}\).
Hence \(\mathcal Q_{t,s}\) is the horizontal barycenter.

\medskip
\noindent
\textbf{Step 4: Vertical slice barycenter identity.}

Fix \(t\in[0,1]\). Since \(\Phi_{t,s}\) is translation by \(t(m_1-m_0)\),
\[
C_t=(\Phi_{t,0})_\#\rho(0)
=\mathcal N(m_0+t(m_1-m_0),\Sigma),
\]
and
\[
D_t=(\Phi_{t,0})_\#\rho(2)
=\mathcal N(m_2+t(m_1-m_0),\Sigma).
\]
Their weighted \((1-s,s)\)-barycenter is the Gaussian with covariance
\(\Sigma\) and mean
\[
(1-s)\bigl(m_0+t(m_1-m_0)\bigr)
+s\bigl(m_2+t(m_1-m_0)\bigr)
=
m_0+t(m_1-m_0)+s(m_2-m_0),
\]
again equal to the mean of \(\mathcal Q_{t,s}\).
Thus \(\mathcal Q_{t,s}\) is also the vertical barycenter. This proves the theorem.
\end{proof}
%\end{document}

\section{Additional Experimental Details and Results}
\label{sec:additional_results}

%\subsection{Additional Wasserstein Barycenter Results on Complex Distributions}
%\label{sec:appendix_wb_complex}

\begin{figure}[h]
\centering
\includegraphics[width=0.95\columnwidth]{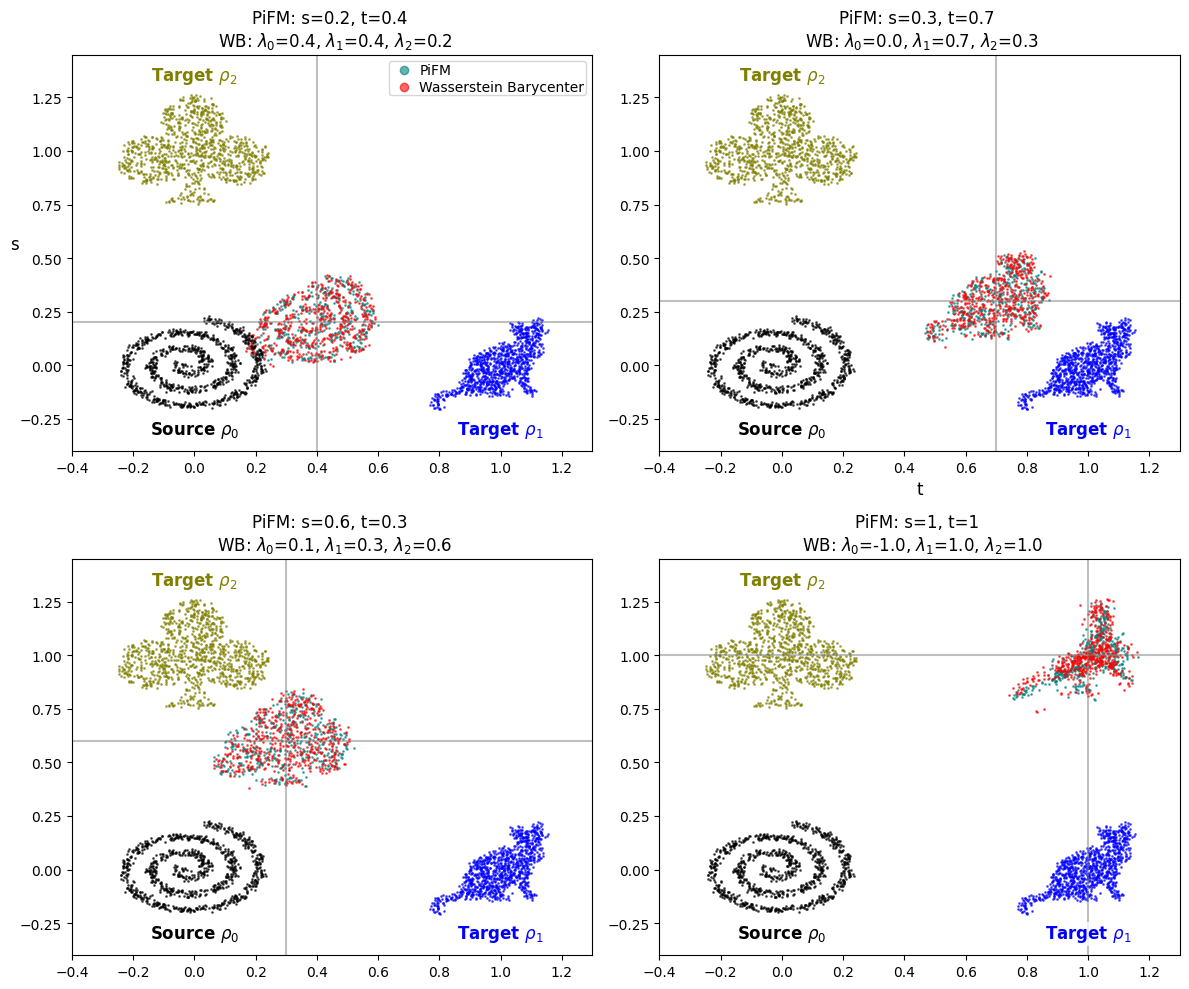}
\caption{\small \textit{Comparison between output distributions from PiFM model (ours) and Wasserstein Barycenters for a) within and  b) outside  the probability simplex for spiral, clover and cat distributions . The predicted distributions from the PiFM model correspond to the Wasserstein Barycenters.}}
\label{fig:PFM_vs_WB_csc}
\end{figure}

\subsection{Comparison of execution times for PiFM and Free Support Algorithm (PyOT)}\label{sec:comp_WPOT}
it is worth noting that computing Wasserstein barycenters is very expensive, even with the fast free-support computation implementation in the POT package. We compare execution time between the former and show that our method considerably reduces the time required to compute the barycenter (Supplementary Figure~\ref{fig:PFM_vs_WB_time}) and state that PiFM can be leveraged for fast inference of Wasserstein barycenters, with most of the computational cost incurred during model training. 

\begin{figure}[H]
\centering
\includegraphics[width=0.8\columnwidth]{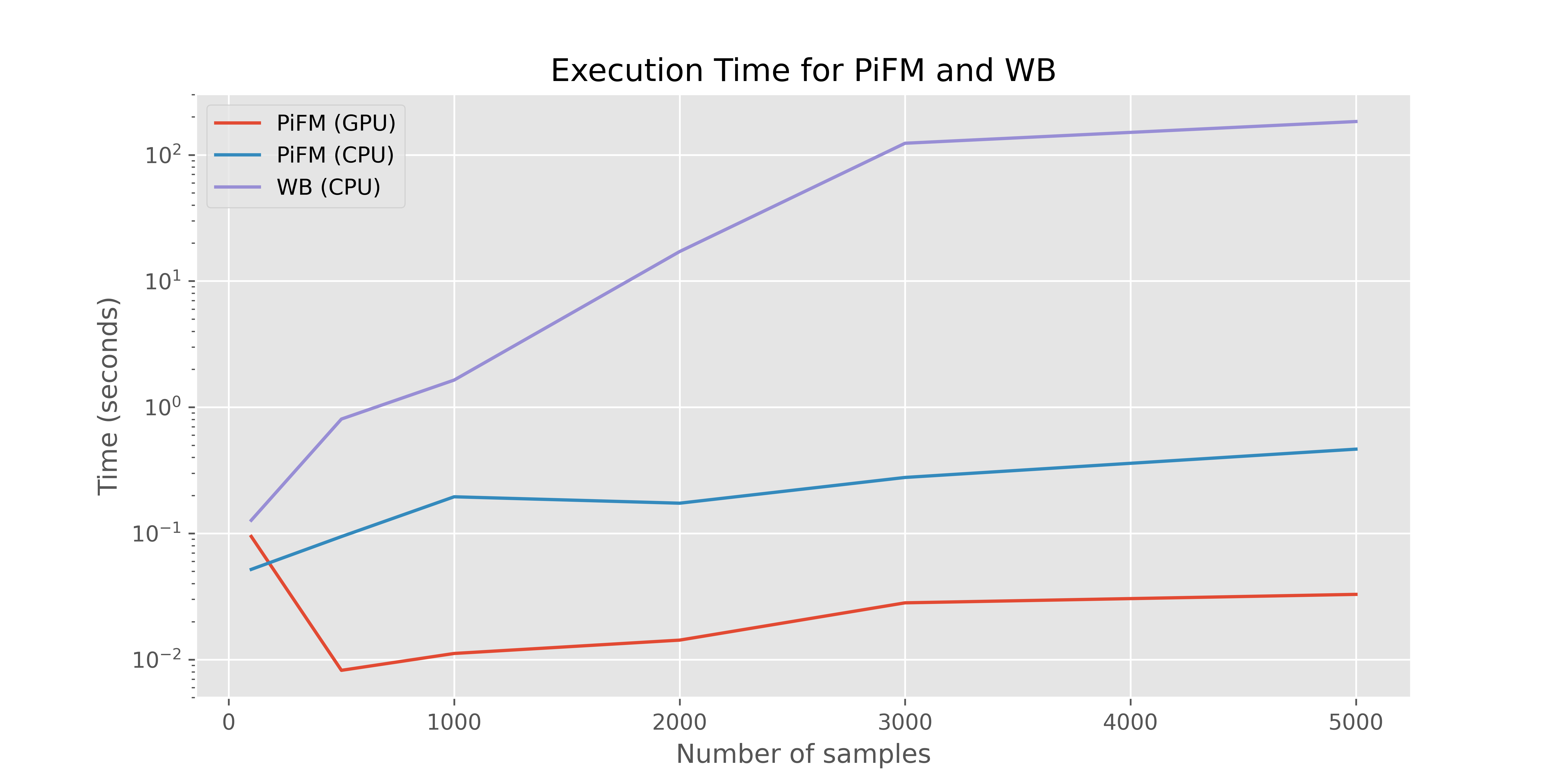}
\caption{Comparison of execution times between free support barycenter (POT package) and PiFM (ours).}
\label{fig:PFM_vs_WB_time}
\end{figure}

\subsection{Supplementary Details and Results for Unsupervised Image-to-Image Translation}

\subsubsection{Dataset, Preprocessing and Model architecture for Image-to-Image Translation}\label{sec:Model_architecture}

We train and evaluate on the CelebA dataset \citep{CelebA}. From the official split we extract images and attribute labels, and use two attributes in the joint experiment: \texttt{Smiling} (attribute~A) and \texttt{Black\_Hair} (attribute~B). Each image is center-cropped to 178, resized to $64\times64$, converted to a tensor and normalized to $[-1,1]$ (per-channel mean $0.5$, std $0.5$). For training we partition the dataset into four subsets by attribute presence: \emph{sources} (neither A nor B), \emph{A-only} (has A, not B), \emph{B-only} (has B, not A), and \emph{both}.

For model architecture, We use a shared UNet backbone with two per-pixel MLP heads implemented as $1\times1$ convolutions (equivalent to applying a small MLP independently at every spatial location) with $SiLU$ non-linear activation function. Each of these two MLPs outputs the vector field corresponding to the flow or progression along each of the attributes mentioned above. Details of the model architecture used in this experiment are summarized in Table~\ref{tab:model_architecture_pfm} in Appendix~\ref{sec:additional_results}.

\begin{table}[H]
\caption{Summary of the model architecture used in the unsupervised image-to-image translation experiment (Section~\ref{sec:Training_PFM_ImageTranslation}).}
\centering
\scriptsize
\setlength{\tabcolsep}{3pt}
\renewcommand{\arraystretch}{1.2}

\begin{tabular}{@{} l p{3.2cm} c p{2.2cm} p{3.2cm} @{}}
\toprule
\textbf{Subnetwork} & \textbf{Component} & \textbf{\# PWE} & \textbf{Input / Output} & \textbf{Details} \\
\midrule

UNet backbone
& Encoder (downpath): [128 $\to$ 256 $\to$ 256 $\to$ 256], 2 ResBlocks per level
& 2\textsuperscript{*}
& 3 / 128
& Downsampling encoder with residual blocks per stage \\

& Bottleneck (mid): 256 channels with attention
& 2\textsuperscript{*}
& 128 / 128
& Residual blocks + attention at mid-resolution \\

& Decoder (uppath): [256 $\to$ 256 $\to$ 128 $\to$ 128]
& 2\textsuperscript{*}
& 128 / 128
& Symmetric upsampling with skip connections \\

\midrule

Per-pixel MLP (Head\_h)
& Conv MLP: 128 $\to$ 128 $\to$ 3 (Conv1x1 + SiLU + Conv1x1)
& 2
& 128 / 3
& Predicts per-pixel field $v(t,x)$ \\

Per-pixel MLP (Head\_v)
& Conv MLP: 128 $\to$ 128 $\to$ 3 (Conv1x1 + SiLU + Conv1x1)
& 2
& 128 / 3
& Predicts per-pixel field $v(t,x)$ \\

\midrule

Prediction semantics
& --
& --
& (B,3,H,W)
& Output is per-pixel displacement field $v(t,x)$ \\
\bottomrule
\end{tabular}

\vspace{0.5em}
\begin{flushleft}
\footnotesize{\textsuperscript{*} ``\# PWE'' indicates number of residual blocks per level.}
\end{flushleft}
\label{tab:model_architecture_pfm}
\end{table}

\subsubsection{Meta Flow Matching (MFM) as a baseline for unsupervised image-to-image translation}
{
    The original meta flow matching (MFM) \citep{atanackovic2025meta} work was developed and evaluated on low-dimensional point-cloud tasks (letters, trellis) where each training environment is a compact 2D point distribution and the model architecture is a GNN + flow decoder tailored to that setting. To test whether the MFM method, the alternating meta-style optimization that jointly learns per-environment embeddings and a conditional flow decoder, can scale to high-dimensional image data, we adapted the MFM pipeline to the CelebA \citep{CelebA} image dataset. Our goal was to keep the MFM training dynamics intact while replacing dataset- and architecture-specific pieces so that the same meta-learning principle can operate on images and produce composed transformations (e.g., smiling + aging).
    
    Point-cloud environments and images are different in data dimensionality, spatial structure, and inductive biases. GNNs operate on unordered point sets and they are a suitable for letters, but they are not appropriate for dense image signals. To respect image structure and give the flow decoder expressive local features and to have a fair comparison with what we proposed in our PiFM pipeline, we replaced the original encoder/decoder backbone with a the same image backbone we used in our PiFM pipeline (architecture details are presented in Table~\ref{tab:model_architecture_pfm}: a U-Net-style feature extractor paired with small per-task MLP heads. This preserves the essential MFM decomposition: an encoder that produces a per-environment representation and a decoder that predicts flow fields conditioned on that representation, while providing the right inductive bias for images. Moreover, similar to MFM, we used two-optimizer alternating training: a flow optimizer that updates the decoder MLP heads to match conditional flow targets and an encoder optimizer that updates the backbone embeddings (the meta parameters). Alternating these updates encourages the encoder to produce embeddings that generalize across environments while the decoder learns to follow them. All other core MFM choices (the flow-matching objective, sampling of interpolation times, EMA for evaluation snapshots, and the alternating schedule itself) were left intact so that any differences in outcome can be attributed to the change in data domain and architecture rather than to changes in the meta-method. At training, a practical issue when moving to images is how to pair source images (neither attribute) with target images (A-only or B-only) for supervised flow targets. We used nearest-neighbor (OT-style pairing) on flattened image features (or simple L2 in pixel space for diagnostics) to create target displacements. At inference we generate composed A+B outputs by integrating the sum of the two learned vector fields (or by sequential integrations $t\rightarrow$ or $s\rightarrow$), which mirrors the composition strategy used in our PiFM experiments.
    
    We kept the same evaluation philosophy: visualize sample trajectories (frames along integration paths) and compute distribution distances between predicted distributions and target distributions using the same metrics used for point-cloud experiments. In Figure \ref{fig:MFM_Results}, we provide grids of intermediate frames for the three integration strategies (t-then-s, s-then-t, diagonal/combined) so readers can see whether composition produces plausible joint-attribute images. As shown, MFM fails to modify the source image along either one or both flow axes.

    \paragraph{Training and sampling.} We train two conditional vector fields simultaneously: (i) $v_h(t,x)$, vector field for attribute A (Smiling), and (ii) $v_v(s,x)$, vector field for attribute B (Black\_Hair).        
    On each training step we sample a minibatch of source images $\{x_a\}$ (from \emph{sources}), target-A images $\{x_b\}$ (from \emph{A-only}) and target-B images $\{x_c\}$ (from \emph{B-only}). Each source is paired to a semantically similar target in the corresponding target minibatch using mini-batch discrete optimal transport (EMD) on flattened pixel vectors; denote the aligned targets by $x_b^{\text{aligned}}$ and $x_c^{\text{aligned}}$. We optimize the network with Adam (base learning rate of $2\!\times\!10^{-4}$), gradient clipping of 1.0 and a linear warm-up schedule. An exponential moving average (EMA, decay $0.999$) of the model parameters is maintained and used for visualization similar to \citep{tong2023improving}.

    \vspace{-5pt}
        
    \paragraph{Inference.} Let $N_{\text{steps}}$ be the integer number of update steps chosen for the inference (we use Euler updates of step size $1/N_{\text{steps}}$). The three flow integration strategies implemented are: (i) \textit{Integrate $t$ then $s$ ($t\!\to\!s$)}: First integrate the A-field from $t=0$ to $1$, then integrate the B-field from $s=0$ to $1$. We split the total number of steps into two parts (approximately half for the first phase, half for the second) so the total frame count equals $N_{\text{steps}}+1$. At each sub-step we update:$        
    x \leftarrow x + \frac{1}{N_{\text{phase}}}\, v_h(t,x)
    \quad\text{or}\quad
    x \leftarrow x + \frac{1}{N_{\text{phase}}}\, v_v(s,x),
    $
    depending on the phase. (ii) \textit{Integrate $s$ then $t$ ($s\!\to\!t$)}: Same as above but apply the B-field first and the A-field second. Order-of-application is intentionally swapped to reveal commutativity. (ii) \textit{Integrate diagonally ($t=s$)}: Integrate along the diagonal $t=s=\tau$ from $0$ to $1$. At each step evaluate both fields and take their sum:
    $
    v(\tau,x) = v_h(\tau,x) + v_v(\tau,x),\qquad
    x \leftarrow x + \frac{1}{N_{\text{steps}}}\, v(\tau,x).
    $
    This approximates simultaneous application of both vector fields.

}

\newpage

\begin{figure}[H]        
    \centering    
    \centerline{\includegraphics[width=0.92\columnwidth]{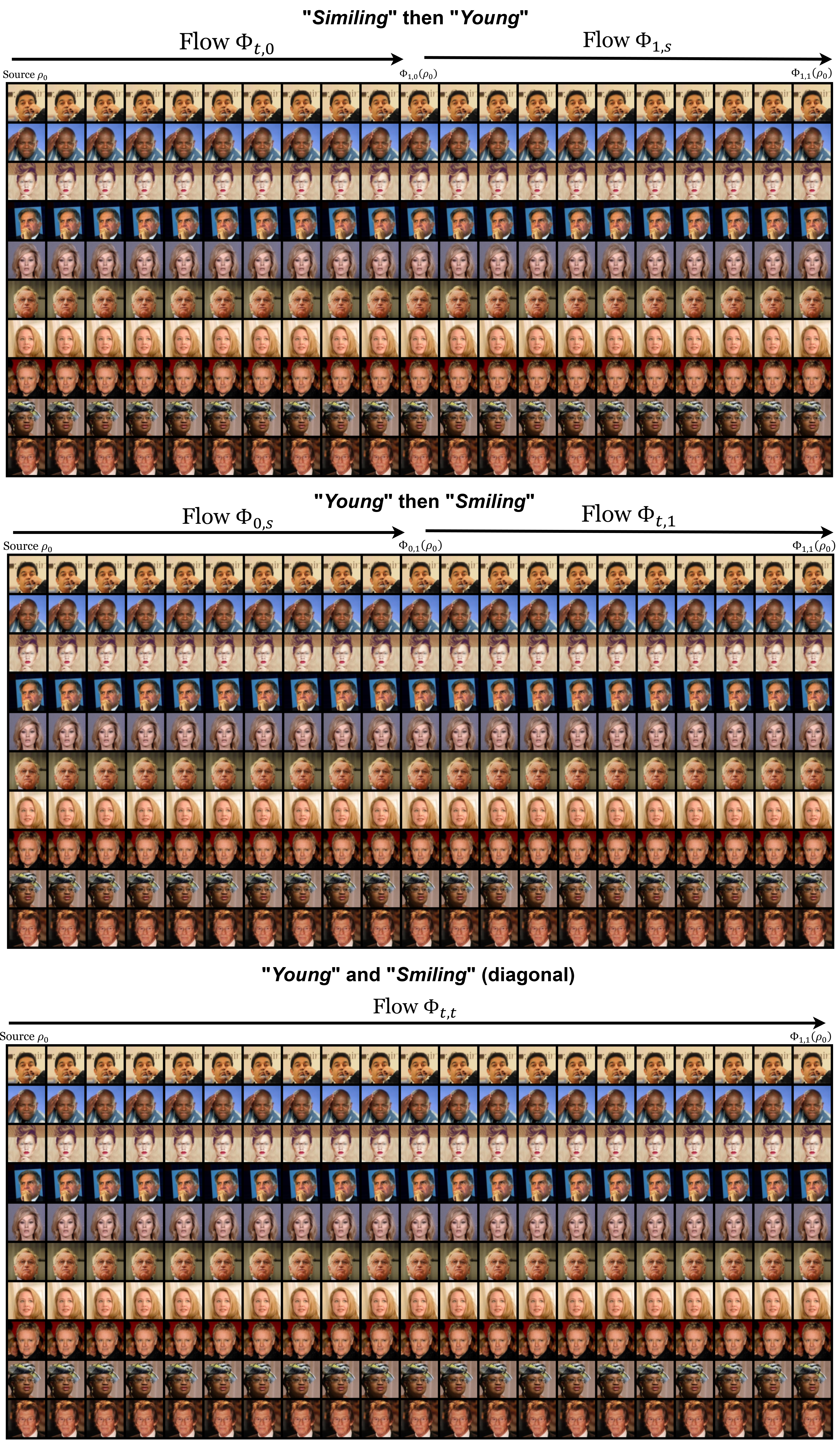}}
    \caption{\small Image-to-image translation results of Meta Flow Matching (MFM) on CelebA \citep{CelebA} dataset across three different flow integration strategies. The A-field (\texttt{Smiling}) produces localized mouth/cheek edits consistent with smiling, while the B-field (\texttt{Young}) darkens and re-textures the hair region. As shown, MFM fails to modify the source image along either one or both flow axes.}
    \label{fig:MFM_Results}
\end{figure}

\begin{figure}[H]        
    \centering    
    \centerline{\includegraphics[width=0.95\columnwidth]{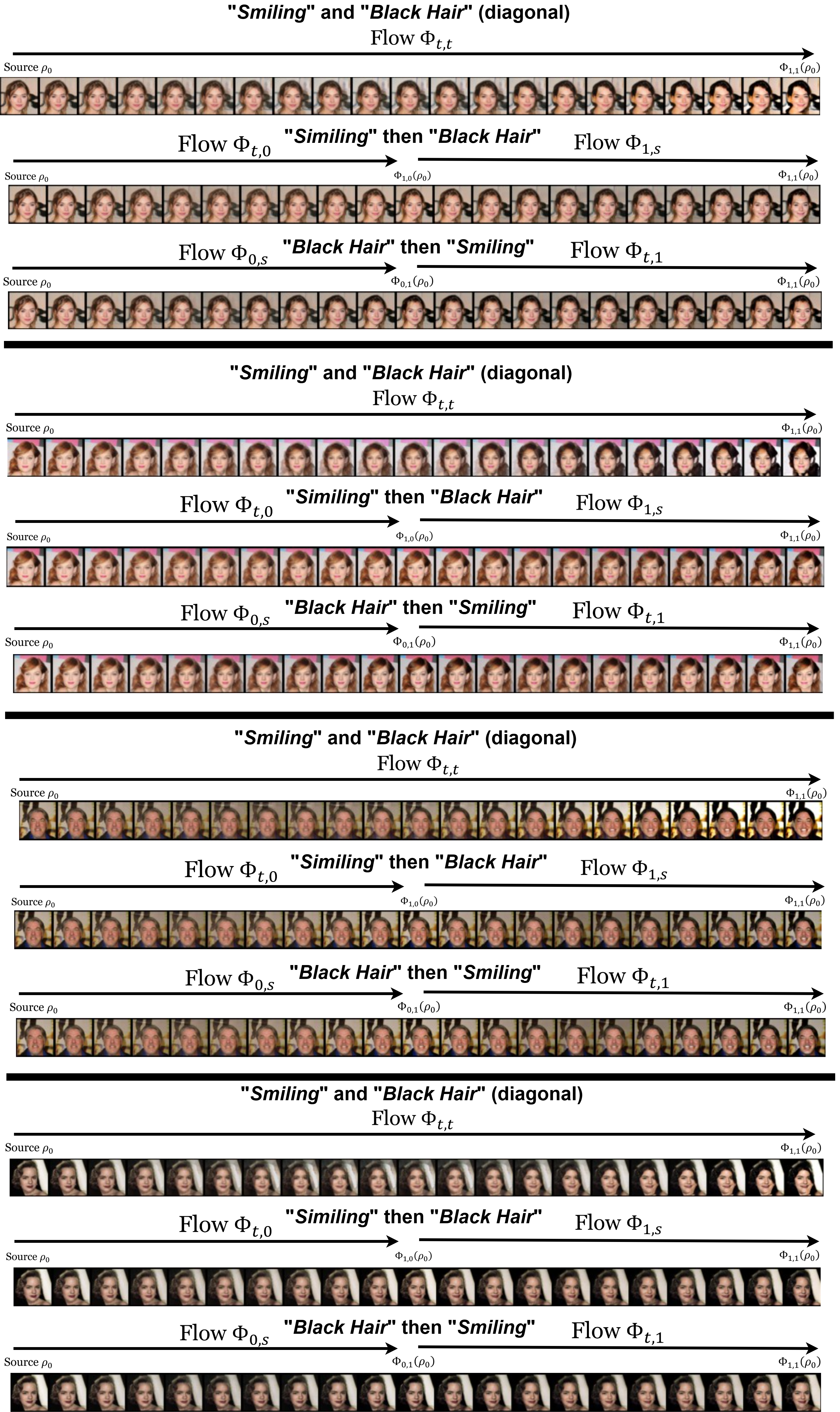}}
    \caption{\small {More results on the Image-to-image translation task for PiFM on CelebA \citep{CelebA} dataset across three different flow integration strategies. The A-field (\texttt{Smiling}) produces localized mouth/cheek edits consistent with smiling, while the B-field (\texttt{Black Hair}) darkens and re-textures the hair region.}}
    \label{fig:PFM_More_Results}
\end{figure}

\subsubsection{FID Evaluation of Commutativity}
\label{sec:appendix_fid}

Standard image fidelity metrics such as FID are not well suited for evaluating PiFM on the commutativity task (Figure~\ref{fig:PFM_CelebA_JointAttr_Smile_BlackHair}), as baseline methods often produce images that are close in pixel space without achieving the desired attribute changes. Figure~\ref{fig:FID_experiment} illustrates this limitation by applying FID to the experiment in Figure~\ref{fig:PFM_CelebA_JointAttr_Smile_BlackHair} in an attempt to quantify commutativity.

\begin{figure}[H]
    \centering
    \includegraphics[width=0.8\linewidth]{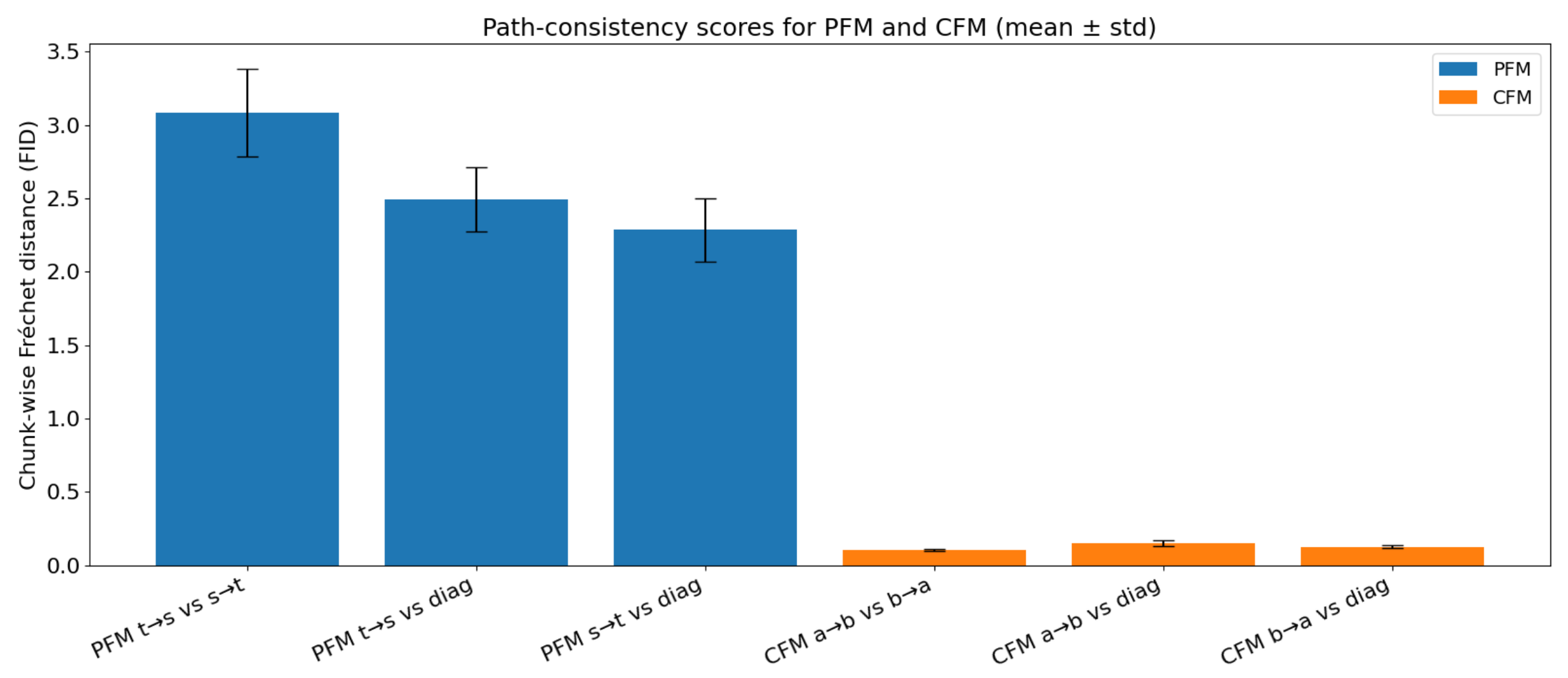}
    \caption{Quantitative experiment for commutativity of PiFM using FID, comparison with CFM. Based on the experiment in Section~\ref{sec:Training_PFM_ImageTranslation}.}
    \label{fig:FID_experiment}
\end{figure}

\subsection{Single-Cell Experiment}
\subsubsection{Data Source and Preprocessing}
\label{sec:single-cell-preprocessing}
We use the public MEF-to-iPSC reprogramming single-cell RNA-seq dataset~\citep{schiebinger2019optimal}, using the serum-condition H5AD file from Figshare article 20051492. The raw object is already normalized and log-transformed. We keep cells from the serum condition, downsample to at most 1,500 cells per experimental day, and focus on the early reprogramming window from day 0 to day 6 at half-day resolution. This yields 19,492 cells after filtering.

For the state representation, we select up to 3,000 highly variable genes, scale expression values with maximum value 10, and compute a 50-dimensional PCA embedding. We use this PCA representation as the input space for PiFM; PHATE embeddings are computed only for visualization. The two biological axes are experimental day and the provided \texttt{Pluripotency} signature score. Experimental day represents chronological reprogramming, while the pluripotency score measures acquisition of an iPSC-like identity. For balanced grouping and diagnostics, we additionally split the pluripotency score within each day into low, middle, and high tertiles.

\subsubsection{Evaluation of Prediction and Commutativity}
\label{sec:single-cell-evaluation}
We evaluate whether the two integrated composition orders produce consistent endpoint distributions, and whether these endpoints occupy the same region as the empirical target distribution. The source distribution is always the day 0, low-pluripotency group, and the comparison is performed over 39 endpoint groups, corresponding to 12 day intervals and 3 pluripotency-bin intervals. To make distances comparable across groups with different source-target scales, we report centroid distance and sliced Wasserstein distance normalized by the corresponding source-to-target distance.

As shown in Table~\ref{tab:singlecell-metrics}, the two composition orders produce endpoint distributions that are close to each other relative to the overall source-target scale. Comparing each endpoint to the empirical target further shows that both compositions land in the same broad target regional, though they are not exactly identical. This imperfect agreement is expected because each endpoint is obtained by composing learned vector fields over relatively long day and pluripotency intervals, so small local integration and approximation errors can accumulate along the trajectory.

\begin{table}
\caption{Target distribution consistency measured by normalized centroid distance and normalized sliced Wasserstein distance. The metrics evaluate whether the two composition orders produce similar endpoint distributions and whether each endpoint lies in the same target region as the empirical distribution.}
\centering
    \begin{tabular}{lrr}
    \toprule
    Comparison & Normalized centroid & Normalized $W_2$ \\
    \midrule
    day then Pluripotency vs true & 0.6979 $\pm$ 0.1263 & 0.7917 $\pm$ 0.1684 \\
    Pluripotency then day vs true & 0.5651 $\pm$ 0.1255 & 0.6601 $\pm$ 0.1691 \\
    day then Pluripotency vs Pluripotency then day & 0.3459 $\pm$ 0.0572 & 0.3465 $\pm$ 0.0560 \\
    \bottomrule
    
    \label{tab:singlecell-metrics}
    \end{tabular}
\end{table}

\subsection{Comparison of distributions obtained from PiFM (ours) and Wasserstein barycenters using Free Support Algorithm (PyOT) for different values t, s}

\begin{figure}[H]
%\vskip 0.2in
\centering
\centerline{\includegraphics[width=0.8\columnwidth]{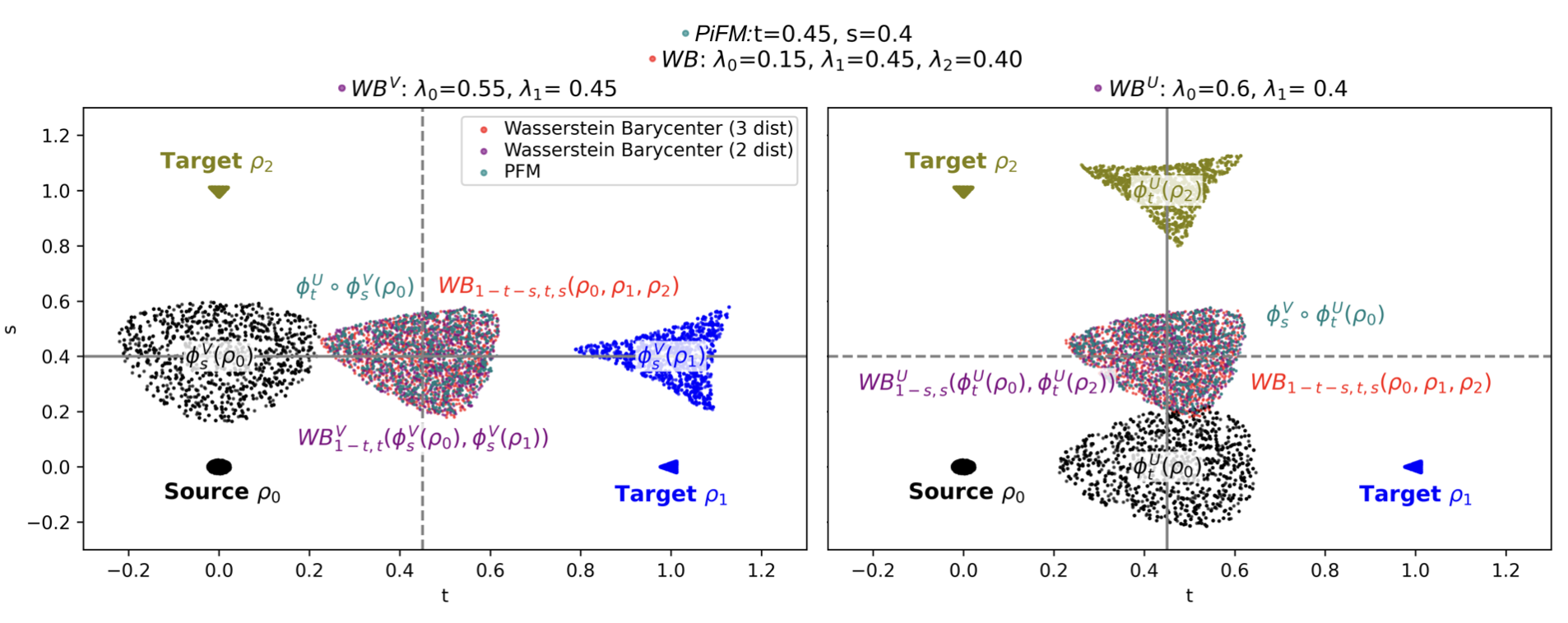}}
\caption{\small \textit{Comparison between output distributions from PiFM (ours), global Wasserstein barycenter and Wasserstein barycenters between $\phi_s^V$ (left) and $\phi_t^U$ (right) projections for t=0.45, s=0.4. Source and target distributions are rescaled for clarity.}}
\label{fig:PV_vs_PU_vs_WB}
\end{figure}

\subsection{Comparison of distributions obtained from PiFM (ours), Wasserstein barycenters using Free Support Algorithm (PyOT) between source and targets and vertical and horizontal Wasserstein barycenters}

In the following, we define $\phi_{t, s=0}(\rho) $ as the output distribution from the push forward operation $(\phi^U_t)_\#$ of $\rho$ with fixed $s=0$, i.e the horizontal integration of the vector field for $t' \in [0,t]$ and $s=0$. Similarly, we define $\phi_{t=1, s}(\rho) $ as the output distribution from the push forward operation $(\phi^V_s)_\#$ of $\rho$ with fixed $t=1$, i.e the integration of the vector field for $s' \in [0, s]$ and $t=1$. 

\begin{figure}[H]
    \centering
    \includegraphics[width=1\columnwidth]{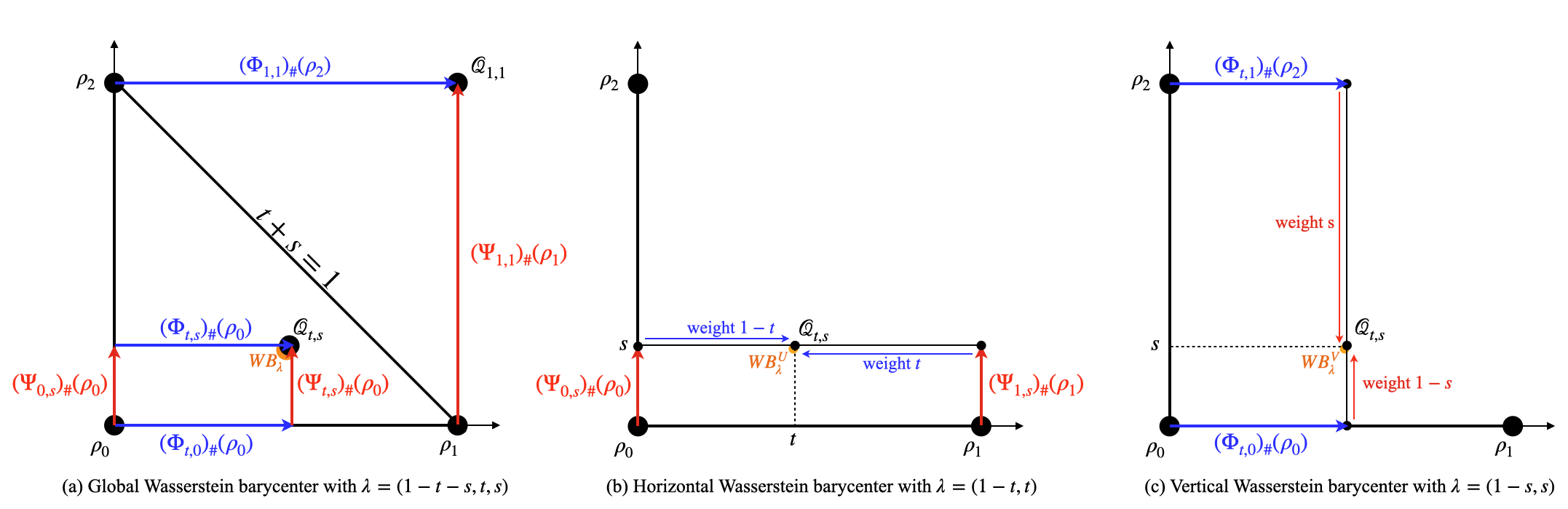}
    \caption{a) the global $(t, s)$-diagram with the simplex $\Delta = \{(t, s) \in [0,1]:  t+s \leq 1\} $ and commuting paths. b) the horizontal slice for fixed s. c) : the vertical slice for fixed t.}
    \label{fig:WB_U_V_explained}
\end{figure}

Supplementary Figures~\ref{fig:Horizontal_Verttical_WB_PiFM_csc} and~\ref{fig:Horizontal_Verttical_WB_PiFM_csc2} provide additional comparisons at different values of $(t,s)$, showing that the PiFM output agrees with the global Wasserstein barycenter and with the corresponding horizontal and vertical slice barycenters.

\begin{figure}[H]
\centering
\includegraphics[width=0.65\columnwidth]{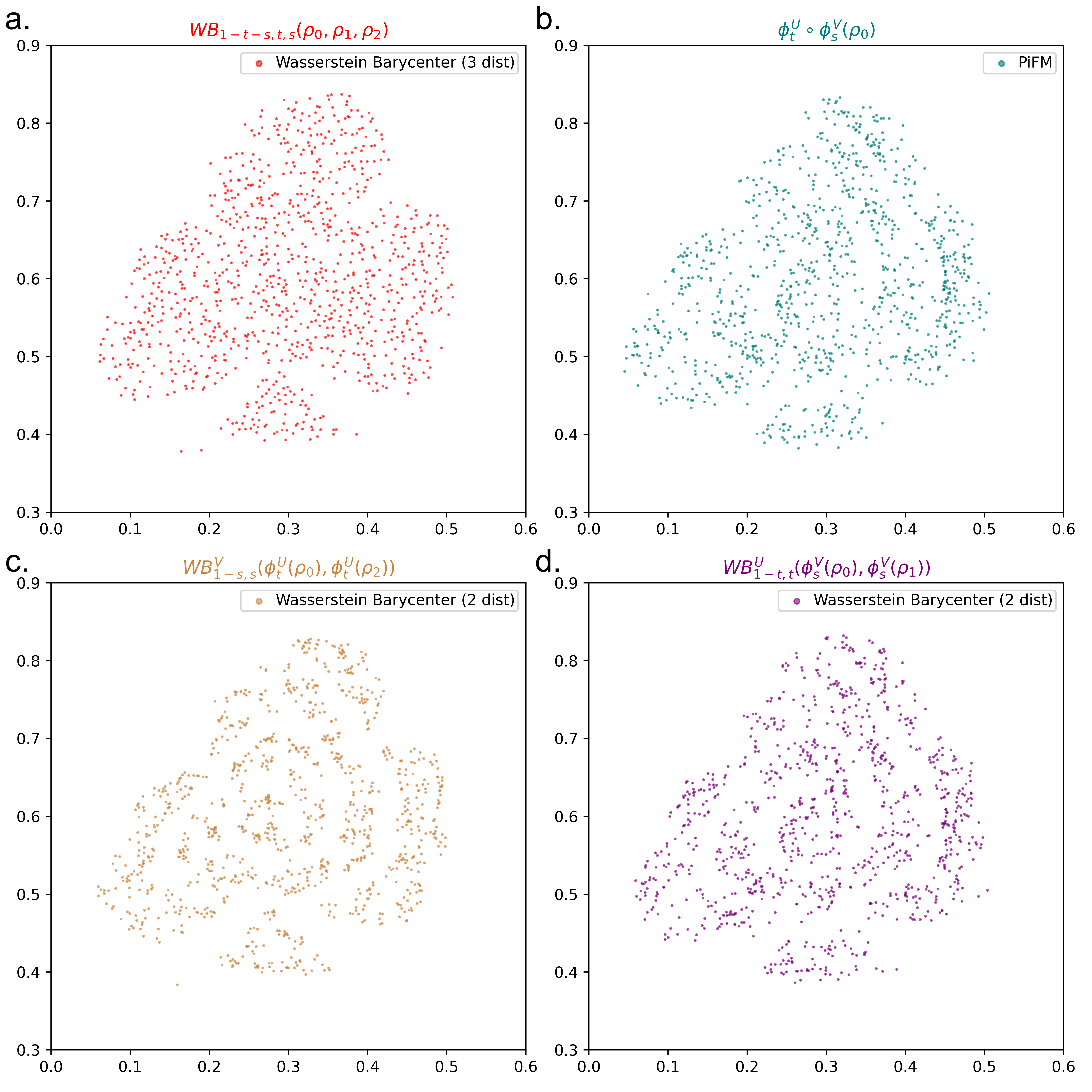}
\caption{\small \textit{Comparison between output distributions from a) Wasserstein Barycenters with weights $(1-t-s, t, s) = (0.1, 0.6, 0.3) $ (using Free Support Algorithm) b) PiFM model (ours) with $t=0.6, s=0.3$, c) vertical Wasserstein barycenter with  $\lambda = (1-t, t)  = (0.4, 0.6)$ computed using Free Support Algorithm of $\phi_{t, s=0}(\rho_0)$ and $\phi_{t, s=1}(\rho_2)$ with $t=0.6$, d) horizontal Wasserstein barycenter  $ \lambda= (1-s, s) = (0.7, 0.3) $ computed using Free Support Algorithm of $\phi_{t=0, s}(\rho_0)$ and $\phi_{t=1, s}(\rho_1)$ with $s=0.3$}}
\label{fig:Horizontal_Verttical_WB_PiFM_csc}
\end{figure}

\begin{figure}[H]
\centering
\includegraphics[width=0.65\columnwidth]{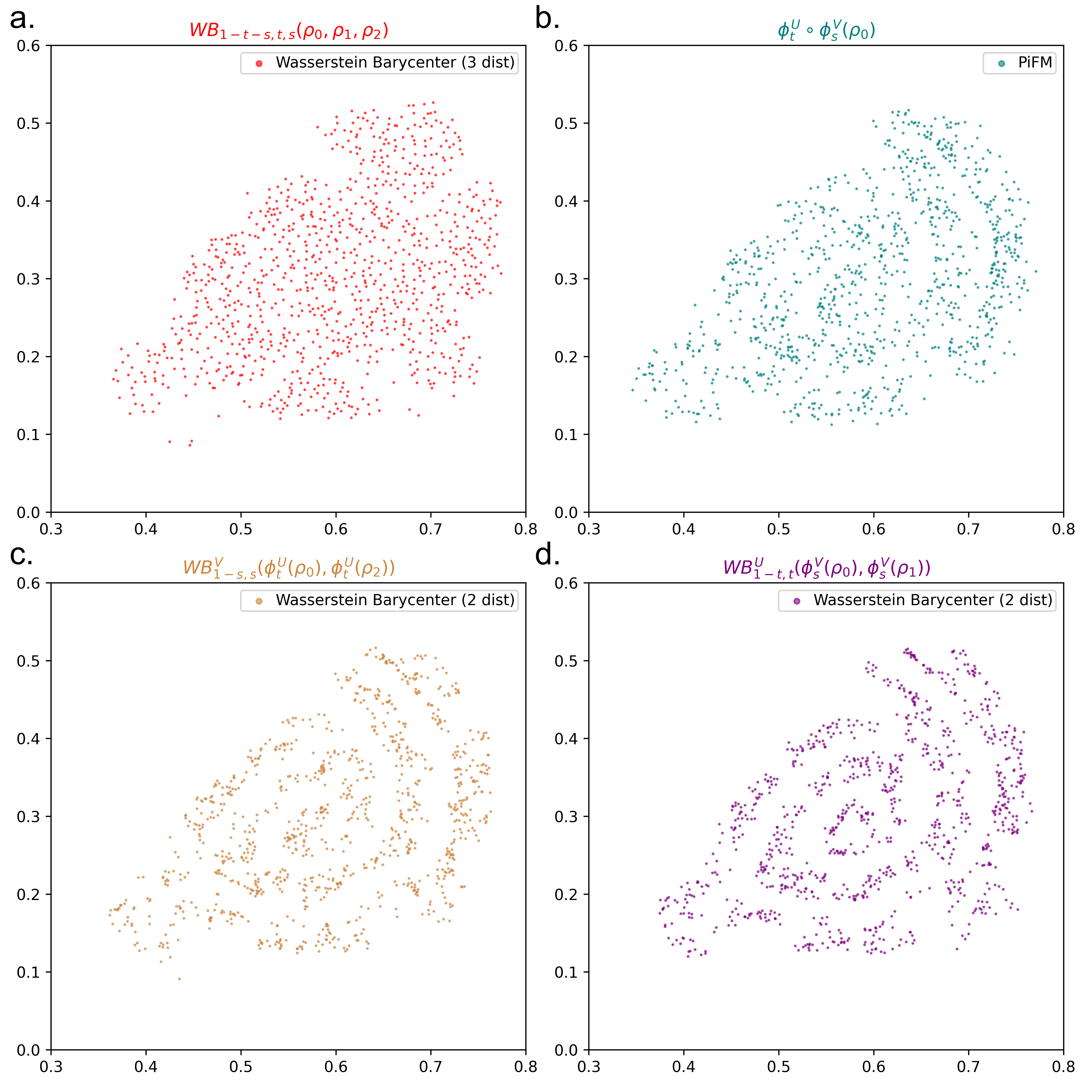}
\caption{\small \textit{Comparison between output distributions from a) Wasserstein Barycenters with weights $(1-t-s, t, s) = (0.1, 0.3, 0.6) $ (using Free Support Algorithm) b) PiFM model (ours) with $t=0.3, s=0.6$, c) vertical Wasserstein barycenter with $\lambda = (1-t, t) = (0.7, 0.3)$ computed using Free Support Algorithm of $\phi_{t, s=0}(\rho_0)$ and $\phi_{t, s=1}(\rho_2)$ with $t=0.3$, d) horizontal Wasserstein barycenter $\lambda= (1-s, s) = (0.4, 0.6)$ computed using Free Support Algorithm of $\phi_{t=0, s}(\rho_0)$ and $\phi_{t=1, s}(\rho_1)$ with $s=0.6$}}
\label{fig:Horizontal_Verttical_WB_PiFM_csc2}
\end{figure}

\subsection{Further Discrepancy Measurements}
In the following subsections, we provide additional empirical evidence supporting the results presented in the main manuscript, along with a small ablation study investigating whether the choice of coupling affects distributional path independence.

\subsubsection{Figure 1}
In the table below we provide a comparison of Wasserstein distance for CFM and PiFM in 3 integration strategies: $t$ then $s$, $s$ then $t$ and diagonal. This experiment is done across 20 seeds.

\begin{table}[H]
    \centering
    \caption{Comparison of the $W_2$ distance across different integration strategies.}
    \label{tab:coupling_ablation}
    \begin{tabular}{lccc}
        \toprule
        \textbf{Model} & \textbf{$t$ then $s$} & \textbf{$s$ then $t$} & \textbf{Diagonal} \\
        \midrule
        CFM  & $1.566972 \pm 0.132423$ & $0.686938 \pm 0.083671$ & $0.954987 \pm 0.091452$ \\
        PiFM & $0.027063 \pm 0.002321$ & $0.019532 \pm 0.001842$ & $0.010326 \pm 0.001174$ \\
        \bottomrule
    \end{tabular}
\end{table}

\subsubsection{Figure 4}

For Figure~4, we run the experiment with varying training batch sizes, obtaining the following results.

\begin{table}[H]
    \centering
    \caption{Comparison of the $W_2$ distance across different integration strategies.}
    \label{tab:curly_fm_ablation}
    \begin{tabular}{lccc}
        \toprule
        \textbf{Model} & \textbf{$t$ then $s$} & \textbf{$s$ then $t$} & \textbf{Diagonal} \\
        \midrule
        Curly-FM         & $0.2808 \pm 0.0418$ & $1.6132 \pm 0.1079$ & $0.7826 \pm 0.0662$ \\
        PiFM (no reg.)   & $0.2154 \pm 0.0536$ & $1.5241 \pm 0.1943$ & $0.9853 \pm 0.1898$ \\
        PiFM (with reg.) & $0.1715 \pm 0.0203$ & $0.5236 \pm 0.0468$ & $0.2857 \pm 0.0899$ \\
        \bottomrule
    \end{tabular}
\end{table}

By varying the random seed while fixing the batch size at 256, we obtain the following results.

\begin{table}[H]
    \centering
    \caption{Comparison of the $W_2$ distance across different integration strategies.}
    \label{tab:curly_fm_ablation_2}
    \begin{tabular}{lccc}
        \toprule
        \textbf{Model} & \textbf{$t$ then $s$} & \textbf{$s$ then $t$} & \textbf{Diagonal} \\
        \midrule
        Curly-FM         & $0.2297 \pm 0.0180$ & $1.5635 \pm 0.1151$ & $0.7285 \pm 0.0611$ \\
        PiFM (no reg.)   & $0.1996 \pm 0.0206$ & $0.9204 \pm 0.1107$ & $0.8559 \pm 0.1446$ \\
        PiFM (with reg.) & $0.1649 \pm 0.1272$ & $0.2606 \pm 0.0533$ & $0.2791 \pm 0.0683$ \\
        \bottomrule
    \end{tabular}
\end{table}

Both experiments are carried out with $n=10$ runs.

\subsubsection{Experiment 4.3}

Recovery of held-out day-by-pluripotency populations in the WOT reprogramming experiment. Values show mean $\pm$ standard deviation across held-out experiments. Distances are normalized by the source-to-target distance, with lower values indicating better recovery and one corresponding to no improvement over the source distribution.

\begin{table}[H]
    \centering
    \caption{Comparison of normalized sliced $W_2$ and MMD$^2$ across different methods.}
    \label{tab:single_cell_metrics}
    \begin{tabular}{lcc}
        \toprule
        \textbf{Model} & \textbf{Normalized sliced $W_2$} & \textbf{Normalized MMD$^2$} \\
        \midrule
        Local OT interpolation baseline & $0.1020 \pm 0.0242$ & $0.0206 \pm 0.0096$ \\
        PiFM: day then pluripotency & $0.3370 \pm 0.0427$ & $0.2121 \pm 0.0672$ \\
        PiFM: pluripotency then day & $0.3108 \pm 0.0416$ & $0.1623 \pm 0.0512$ \\
        \bottomrule
    \end{tabular}
\end{table}

\subsubsection{Ablation for coupling type}

Using the Figure 1 toy example, we compare CFM and PiFM under OT and independent couplings over 10 seeds, measuring mean per-sample path discrepancy $\mathbb{E}[x_{t\to s}-x_{s\to t}]$. CFM remains path-dependent, while PiFM maintains low discrepancy for both couplings.

\begin{table}[H]
    \centering
    \caption{Comparison of the $W_2$ distance for different coupling strategies.}
    \label{tab:coupling_comparison}
    \begin{tabular}{lcc}
        \toprule
        \textbf{Model} & \textbf{OT coupling} & \textbf{Independent coupling} \\
        \midrule
        CFM  & $1.886 \pm 0.234$ & $1.297 \pm 0.520$ \\
        PiFM & $0.023 \pm 0.007$ & $0.031 \pm 0.068$ \\
        \bottomrule
    \end{tabular}
\end{table}

%%%%%%%%%%%%%%%%%%%%%%%%%%%%%%%%%%%%%%%%%%%%%%%%%%%%%%%%%%%%

\end{document}